\documentclass{article}

\usepackage[preprint]{neurips_2026}

\usepackage[utf8]{inputenc}
\usepackage[T1]{fontenc}
\usepackage{hyperref}
\usepackage{url}
\usepackage{booktabs}
\usepackage{amsfonts}
\usepackage{nicefrac}
\usepackage[expansion=true,protrusion=true,final]{microtype}
\usepackage{xcolor}
\usepackage{enumitem}
\usepackage[skip=2pt,font=small]{caption}
\usepackage[skip=2pt,font=footnotesize]{subcaption}
\usepackage{titlesec}
\usepackage{dsfont}

\titlespacing*{\section}{0pt}{8pt plus 2pt minus 2pt}{4pt plus 1pt minus 1pt}
\titlespacing*{\subsection}{0pt}{6pt plus 1pt minus 1pt}{3pt plus 1pt minus 1pt}
\titlespacing*{\paragraph}{0pt}{4pt plus 1pt minus 1pt}{0.6em}

\setlist{itemsep=2pt, topsep=2pt, parsep=0pt, partopsep=0pt}

\renewcommand{\arraystretch}{0.95}

\usepackage{amsmath}
\usepackage{amssymb}
\usepackage{mathtools}
\usepackage{amsthm}

\usepackage{graphicx}
\usepackage{multirow}
\usepackage{listings}
\usepackage[most]{tcolorbox}
\usepackage[export]{adjustbox}
\usepackage[capitalize,noabbrev]{cleveref}
\tcbuselibrary{skins, breakable}
\usepackage{fancyvrb}
\usepackage{float} 

\theoremstyle{plain}
\newtheorem{theorem}{Theorem}[section]

\newtheorem{corollary}[theorem]{Corollary}
\theoremstyle{definition}
\newtheorem{definition}{Definition}

\theoremstyle{remark}

\newtcolorbox{promptbox}[1]{
    colback=white,
    colframe=black,
    title={#1},
    fonttitle=\bfseries,
    breakable,
    skin=enhanced,
}

\newenvironment{minted}[2][]{\VerbatimEnvironment\begin{Verbatim}[fontsize=\footnotesize,frame=single]}{\end{Verbatim}}
\providecommand{\usemintedstyle}[1]{}
\usepackage[disable,textsize=tiny]{todonotes}

\providecommand{\printAffiliationsAndNotice}[1]{}

\title{Internalizing Geometric Law:\\ Learning from Solver Residuals for Precision-Critical Generation}

\author{%
  Rafael Cabral, Pang Zixi, Ziyi Shou, Shen Xin\thanks{Corresponding author.}  \\
  Huawei Celia Team \\
  \texttt{shenxin19@huawei.com}
}

\begin{document}

\maketitle

\begin{abstract}

Large Language Models frequently hallucinate in precision-critical domains such as technical diagramming and mechanical design, where outputs must satisfy strict geometric constraints. We study open-ended geometric synthesis from natural language: translating free-form descriptions into precise constructions whose entities must simultaneously satisfy dozens of interacting constraints. To make this tractable, we release \textit{PyGeoX}, a programmable geometric DSL that compiles declarative constraints into a differentiable loss, and \textit{PyGeoX-Bench}, a stratified suite of 300 problems with per-constraint verifiable rewards. Using PyGeoX as a verifier, we identify a failure mode we call \textit{Outlier Gradient Masking}: under global-norm rewards (any scheme that aggregates residuals through a single norm, for example, $\exp(-\text{MSE})$), a single outlier constraint can nullify the learning signal across all others. To address this, we propose \textit{Saturating Additive Rewards} (SAR), which decompose the reward into bounded per-constraint terms, preserving partial progress and ensuring consistent gradients even under severe violations. Against MSE-based rewards, the natural baseline for geometry solvers, SAR improves the hard-tier solving rate by $2.3\times$, and the resulting 8B model is competitive with much larger frontier systems on this benchmark. We release the engine, benchmark, and data at \url{https://github.com/Huawei-AI4Math/PyGeoX}.

\end{abstract}


\section{Introduction}
\label{sec:introduction}

Large Language Models have achieved remarkable proficiency in semantic tasks, from code synthesis to literature summarization. However, their capability to adhere to rigorous, precision-critical constraints remains brittle. This phenomenon, which we term ``precision hallucination,'' diverges fundamentally from the ``semantic hallucinations'' of earlier models: outputs are syntactically coherent and semantically plausible but violate the exact laws of geometry, physics, or logical consistency \citep{trinh2024alphageometry}.

For researchers aiming to deploy generative models in engineering domains, specifically technical diagramming, computer-aided design (CAD) and kinematic mechanism design, this limitation is existential. A generated technical diagram might look plausible, but if it connects two components with a physically impossible linkage, it constitutes more than a simple hallucination. Instead, it represents a critical functional failure. We consider all such problems under the umbrella of \textit{Geometric Constraint Solving} (GCS), as they involve finding a configuration of geometric entities that satisfies a set of relational constraints.


\begin{figure}[htp]
\centering
\includegraphics[width=\linewidth]{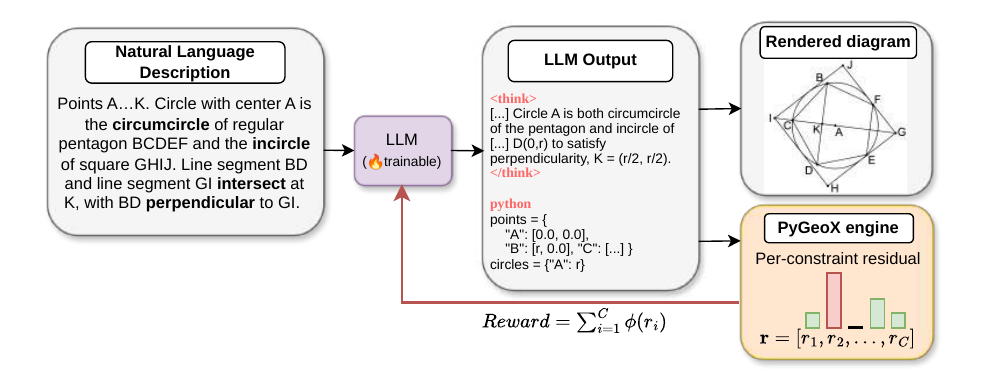}
\caption{Task overview. The agent thinks and emits Python code containing exact coordinates and radii.}
\label{fig:workflow}
\end{figure}
\vspace{-0.25em}

Current approaches to mitigate this limitation often rely on delegating reasoning to an external numerical or symbolic solver. While this ensures strict constraint satisfaction, it restricts the LLM to a translation-only role. The model merely converts natural language into solver-specific syntax, which hinders its ability to internalize spatial logic or perform creative, constrained synthesis. Existing symbolic solvers and geometry constraint engines, such as FormalGeo \citep{zhang2023formalgeo} and the static predicate systems used in AlphaGeometry \citep{trinh2024alphageometry} or FGeo-DRL \citep{fgeodrl}, impose inherent limitations on the expressivity of objects and relationships. For instance, some of these systems cannot deal with inequality relationships (e.g., ``line 1 is larger than line 2''). Crucially, this creates a bottleneck across both pure geometric solving and engineering domains like kinematic synthesis. Current symbolic engines lack the vocabulary to define dynamic functional requirements, which limits their utility to textbook geometry problems rather than real-world engineering design. By delegating all spatial logic to a rigid API, the model never learns the underlying geometric laws, rendering it brittle when faced with novel constraints that lie outside the tool's pre-defined vocabulary.

\paragraph{Our approach.} We formulate end-to-end GCS as a new LLM alignment task and build the infrastructure to train on it: \textit{PyGeoX}, a geometry library with auto-compiled differentiable rewards (Figure~\ref{fig:PyGeoX}), a procedural data pipeline yielding $\sim 100k$ validated training problems, and the benchmarks needed to measure progress. With this infrastructure in place, the remaining question is reward design. Standard Reinforcement Learning with Verifiable Rewards (RLVR) aggregators, such as global norms (e.g., $\exp(-\|\mathbf{r}\|_2^2)$, where $\mathbf{r} = [r_1, \ldots, r_C]$ is the per-constraint residual vector returned by the solver) and binary success indicators, suffer a failure mode we call \textit{Outlier Gradient Masking}: a single severely violated constraint drives the aggregated reward to zero, zeroing out the policy gradient and destroying the learning signal from every satisfied constraint. We propose \textit{Saturating Additive Rewards} (SAR), which sum independent bounded kernels per constraint so that progress on satisfied constraints survives. Figure~\ref{fig:workflow} summarizes the overall framework.

\subsection{Preliminaries}
\label{sec:preliminaries}

\paragraph{Task.} Given a natural-language description $\mathrm{x}$ specifying geometric objects (points, lines, circles, polygons) and relational constraints (incidence, perpendicularity, tangency, length, area, \ldots), the model must reason about the geometry and produce a Python program whose output is a structured dictionary containing exact metric values for every free parameter: point coordinates and circle radii, such that all constraints are simultaneously satisfied. The model operates as a single-turn code agent with access to standard scientific libraries (\texttt{numpy}, \texttt{scipy}, \texttt{sympy}) and may choose any solving strategy: constructive ruler-and-compass derivations, numerical optimization, or hybrids.

\paragraph{Verification.} The PyGeoX symbolic-numeric solver (Section~\ref{sec:pygeox}) compiles the constraints into a residual vector $\mathbf{r} = [r_1, \ldots, r_C]$, where each $r_i \geq 0$ measures the violation magnitude of one constraint (e.g., Euclidean distance for incidence, angular deviation for parallelism). A solution is judged correct when $\|\mathbf{r}\|_2^2 < 10^{-3}$ (allowing for floating-point error propagation through the constraint equations). Crucially, the LLM cannot use the PyGeoX library. PyGeoX is used only off-policy: for procedural training-data generation (Section~\ref{sec:data_gen}) and for computing per-constraint residuals during RL reward evaluation. This separation is deliberate: confining the agent to a fixed domain-specific language (DSL) would defeat the goal of internalizing geometric law, since the agent would learn DSL syntax rather than spatial reasoning.


\paragraph{Why this is a new task.} No existing public work targets this exact input/output signature. Geometry problem-solving models and benchmarks \citep{geoqa, unigeo, pgps9k, zhang2023formalgeo, geoeval, geosense} output scalar answers to textbook geometry problems or proofs; diagram-generation benchmarks \citep{wei2025ggbench} evaluate rendered images; theorem-proving systems  \citep{trinh2024alphageometry, fgeodrl} output discrete proof steps. None require translating natural language geometric descriptions to numerical coordinates directly using LLMs, and therefore, we release a new benchmark for this task, \textit{PyGeoX-Bench}. 

\paragraph{Contributions.} In this work, we introduce \textbf{PyGeoX-RL}, a neuro-symbolic framework that teaches LLMs to \textit{internalize geometric law}. Our contributions are:

\begin{enumerate}
\item \textbf{End-to-end GCS as an LLM alignment task.} We formulate Geometric Constraint Solving as a new alignment problem: natural language to exact metric coordinates verified by per-constraint residuals, a task formulation that no existing benchmark targets and that even frontier LLMs cannot reliably solve out of the box.

\item \textbf{The PyGeoX engine.} We release a programmable geometric environment covering 35 object types and 38 relationships, with auto-compiled differentiable reward functions. PyGeoX serves as both a scalable data generation engine and an RL Gym environment, and supports expressivity (inequality constraints, polygon relationships) absent from prior solvers.

\item \textbf{Data pipeline and benchmarks.} We provide a procedural pipeline yielding $\sim$100k validated training problems, plus PyGeoX-Bench (300 stratified evaluation problems) and PyGeoX-Wild, with 86 out-of-distribution (OOD) diagnostic problems adapted from a published middle-school geometry benchmark.

\item \textbf{Reward design that makes RL on GCS viable.} We identify a failure mode, Outlier Gradient Masking, common to global-norm and sparse rewards, and introduce Saturating Additive Rewards (SAR). To our knowledge, this is the first application of per-constraint solver residuals as a dense reward for autoregressive LLMs. SAR substantially outperforms field-standard sparse and global-norm rewards on hard GCS problems.
\end{enumerate}

\section{Related Work}
\label{sec:related_work}

\textbf{Neuro-Symbolic Geometric Reasoning.} Integrating geometric engines with LLMs mitigates the precision hallucinations of pure neural models. Current approaches generally partition into three categories with specific limitations for RL training. ``Code-as-Reasoning" frameworks like \textit{GeoCoder} \citep{sharma2025geocoder}, \textit{ToRA} \citep{Gou2023ToRAAT}, and \textit{CAD-Llama} \citep{li2025cadllamaleveraginglarge} effectively  reduce the LLM to a translator from natural language into solver-specific syntax, delegating all geometric reasoning to the external engine. Recent work adds RL to this paradigm \citep{recad,rlcad}, but although these systems benefit from solver feedback during training, the LLM still remains a DSL translator. Discrete deductive systems such as \textit{AlphaGeometry} \citep{trinh2024alphageometry} and \textit{FormalGeo} \citep{zhang2023formalgeo} offer logical rigor but are designed for theorem proving in discrete spaces rather than constructive synthesis. Lastly, traditional CAD tools like \textit{SolveSpace} \citep{solvespace} and \textit{Parasolid} \citep{Sears_1991} are engineered for final precision rather than intermediate learning, lacking granular feedback necessary for agent improvement. 
\textbf{Physics-Informed and Constraint-Based Learning.} Our approach aligns theoretically with Physics-Informed Neural Networks (PINNs) \cite{Raissi2017PhysicsID}, which embed physical laws into deep learning by treating PDE residuals as loss functions. While recent works like \textit{PIRF} \citep{yuan2025pirf} have adapted this paradigm to diffusion models, to the best of our knowledge, we are the first to translate this ``residual-as-supervision'' framework into the domain of autoregressive LLMs. Current state-of-the-art reasoning models, such as \textit{DeepSeek-R1} \citep{Guo_2025} and other RLVR approaches \citep{wang2025geometryzero}, typically discard this fine-grained residual information, relying instead on sparse, binary outcome supervision ($r \in \{0, 1\}$). 



\textbf{Generative Geometric Design.} \citet{wang2025magicgeo} target diagram generation from natural language but delegate coordinate computation to an algorithmic solver; the LLM autoformalizes constraints into a formal specification, and the system is training-free. \citet{casey2025aligning} fine-tune LLMs to predict \emph{constraint labels} given primitives with known coordinates, a classification task on existing geometry rather than coordinate generation. In contrast, our LLM directly emits metric coordinates satisfying every constraint, and the solver is used only for verification and reward computation.

\section{Reward Design for GCS}
\label{sec:theoretical_analysis}

In this section we analyze the central challenge of aggregating a high-dimensional residual vector into a scalar reward $\mathcal{R}$ and motivate the SAR family.

\subsection{From Residuals to Rewards}

A source of confusion in this domain arises from conflating the role of \textit{Reward Functions} in Reinforcement Learning with \textit{Loss Functions} in Supervised Learning. One might intuitively assume that because the sum of squared errors (SSE) works well for regression, a simple transformation (such as negation) will serve as an effective reward.

In supervised regression, minimizing a global norm loss like SSE ($\mathcal{L} = \|\mathbf{r}\|_2^2$) is effective because the gradient is directly proportional to the residual ($\nabla \mathcal{L} \propto \mathbf{r}$). Outliers are explicitly instructive, creating large gradients that tell the model exactly which weights to adjust. In RL, the mechanics of the update are fundamentally different. Let $\pi_\theta$ denote the language model policy with parameters $\theta$, mapping a prompt $\mathrm{x}$ to a distribution over generated sequences $\mathrm{y}$, and let $\mathcal{R}(\mathrm{x}, \mathrm{y})$ be a scalar reward that quantifies how well $\mathrm{y}$ satisfies the geometric constraints specified by $\mathrm{x}$. Given prompts $\mathrm{x} \sim \mathcal{D}$ and generated answers $\mathrm{y} \sim \pi_\theta(\cdot \mid \mathrm{x})$, the policy gradient (omitting KL regularization for clarity) is:
\begin{equation*}
\nabla_\theta J(\theta) = \mathbb{E}_{\mathrm{x} \sim \mathcal{D},\, \mathrm{y} \sim \pi_\theta(\cdot \mid \mathrm{x})} \left[ \mathcal{R}(\mathrm{x}, \mathrm{y}) \cdot \nabla_\theta \log \pi_\theta(\mathrm{y} \mid \mathrm{x}) \right].
\label{eq:policy_gradient}
\end{equation*}
The update direction is determined solely by the log-probability gradient $\nabla_\theta \log \pi_\theta(\mathrm{y} \mid \mathrm{x})$, while the reward $\mathcal{R}(\mathrm{x}, \mathrm{y})$ acts only as a scalar multiplier. This creates a critical vulnerability when the reward depends on the residuals through a global aggregation, such as a sum of squared errors $\|\mathbf{r}\|_2^2$, any $\ell_p$ norm, or a transformation thereof (e.g., $\exp(-\|\mathbf{r}\|_2^2)$ or $\exp(-\text{MSE})$). In all such cases, a single large outlier residual $r_i \gg 0$, common in complex GCS tasks, drives the aggregated reward to near-zero. This ``all-or-nothing'' scaling nullifies the entire policy gradient term, masking the agent's partial progress on the remaining constraints (Figure~\ref{fig:reward_analysis}).

\begin{figure}[htp]
    \centering
    \begin{subfigure}[t]{0.48\columnwidth}
        \centering
        \includegraphics[width=\linewidth]{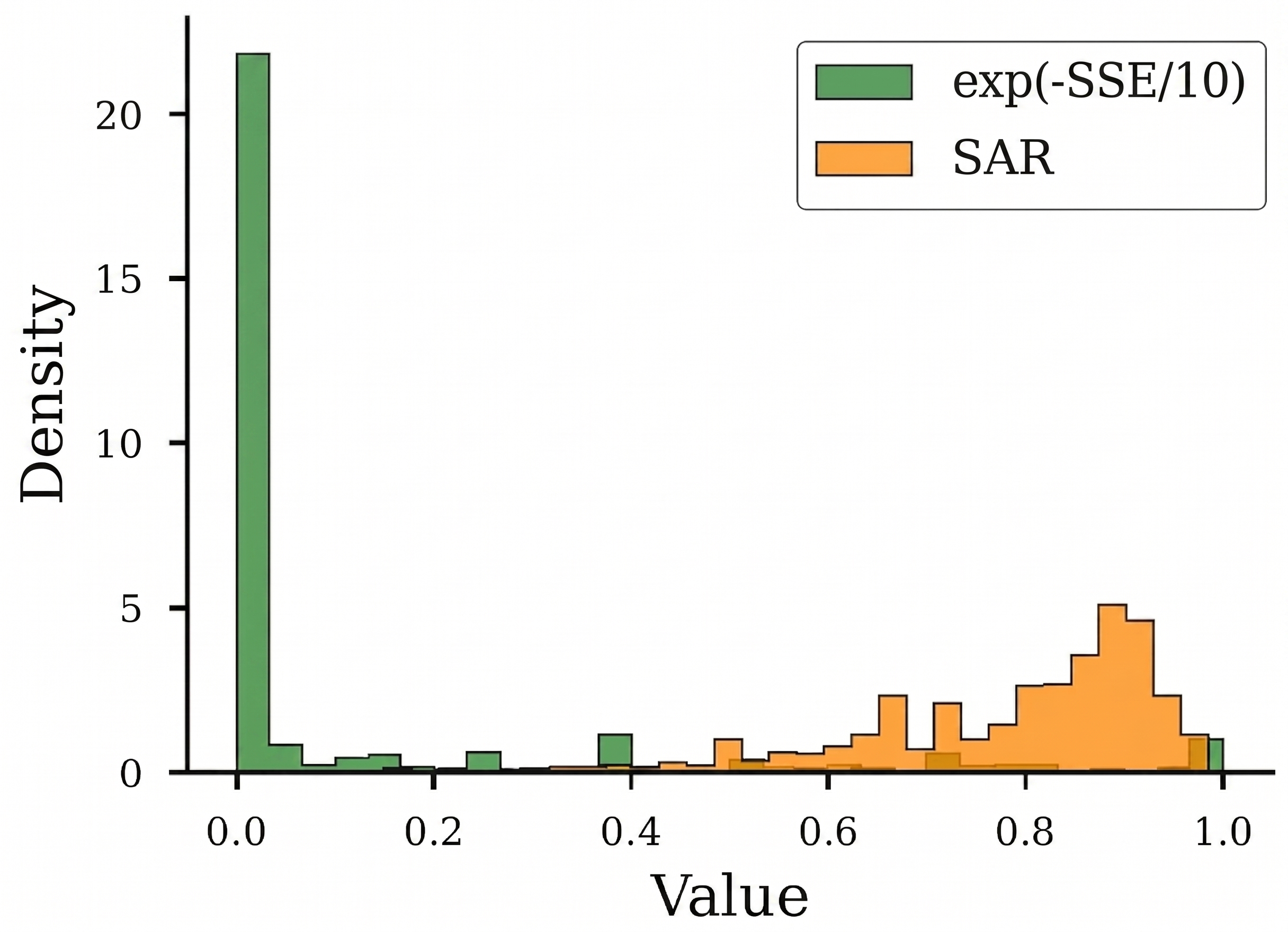}
        \label{fig:partial_progress}
    \end{subfigure}
    \hfill
    \begin{subfigure}[t]{0.46\columnwidth}
        \centering
        \includegraphics[width=\linewidth]{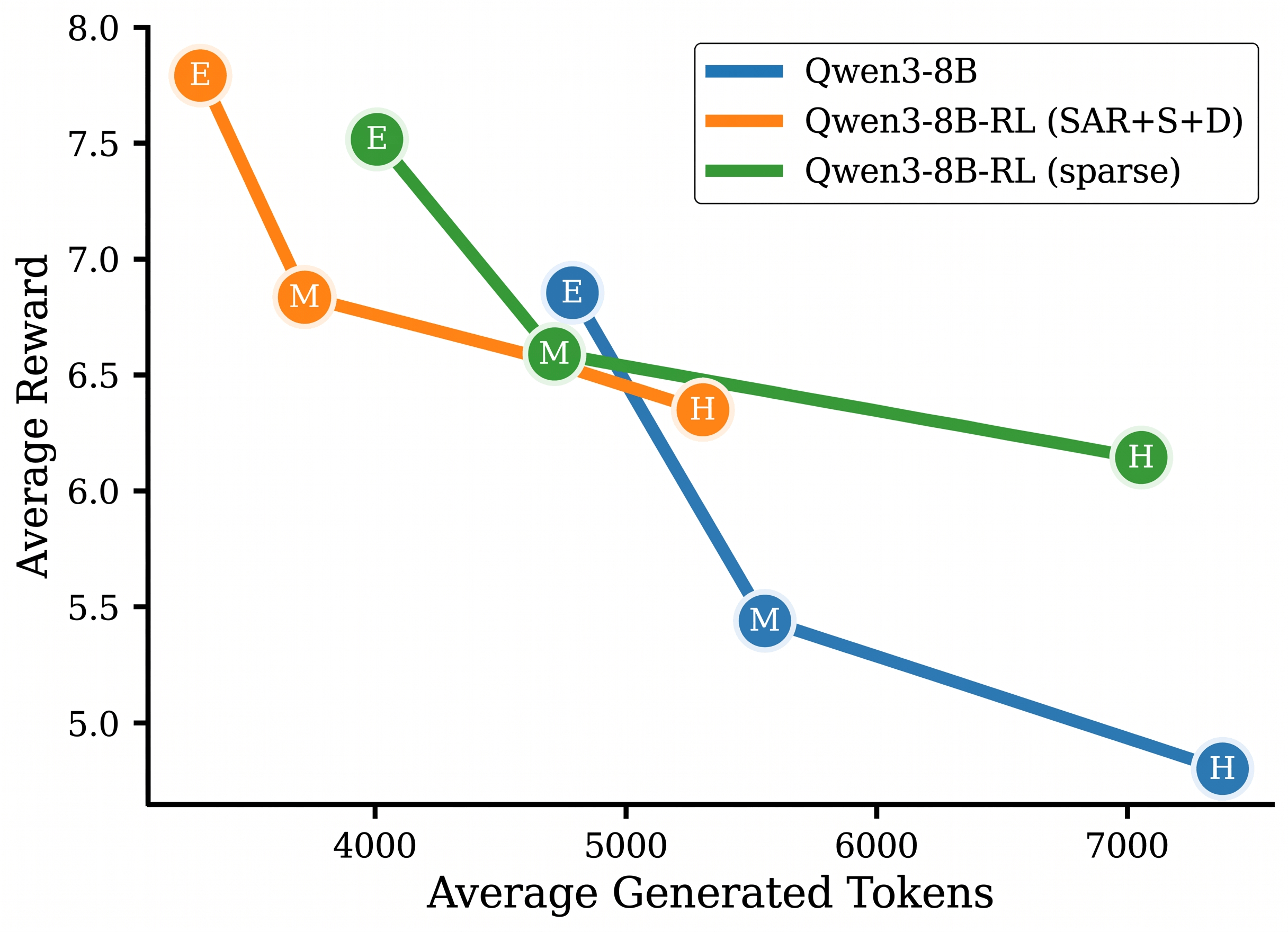}
        \label{fig:reward_vs_tokens}
    \end{subfigure}
    \caption{Reward signal analysis for Qwen3-8B on PyGeoX-Bench. (Left) The global-norm reward collapses to zero on partially correct base-model outputs, providing no learning signal and SAR preserves dense signal even when some constraints remain violated. (Right) SAR-trained models achieve higher reward with fewer tokens across all difficulty tiers (E-Easy, M-Medium, H-Hard) compared to the sparse baseline.}
    \label{fig:reward_analysis}
\end{figure}
\vspace{-0.25em}

\subsection{Saturating Additive Rewards}
\label{subsec:sar}

To resolve outlier gradient masking, we propose summing independent bounded kernels, one per constraint, so that no single residual can collapse the reward signal.

\begin{definition}
\label{eq:sar_def}
A reward function $\mathcal{R}(\mathbf{r})$ is a \emph{Saturating Additive Reward} (SAR) if it decomposes as
\begin{equation}
\mathcal{R}_{\text{SAR}}(\mathbf{r}) = \sum_{i=1}^C \phi(r_i),
\end{equation}
for a monotonically decreasing kernel $\phi : \mathbb{R}_{\geq 0} \to [0, 1]$ with $\lim_{r \to \infty} \phi(r) = 0$.
\end{definition}

We instantiate SAR with the Boltzmann kernel $\phi(r) = e^{-r/T}$ throughout the paper. Other bounded kernels include Cauchy $\phi(r) = 1/(1 + (r/\gamma)^2)$ and sigmoidal $\phi(r) = 1/(1 + e^{k(r-\tau)})$.

We compare global-norm rewards $\mathcal{R}_G = \phi(\|\mathbf{r}\|_p)$ against SAR in two regimes. The first shows global-norm rewards are almost surely zero at random initialization in high dimensions, so RL cannot start. The second shows that even when the agent makes local progress on a single constraint, global-norm rewards mask it whenever the rest of the configuration is disordered, so RL cannot continue. SAR avoids both pathologies. Full statements and proofs are in Appendices~\ref{proof:reward_volume} and~\ref{app:proof_1}.

\paragraph{Signal availability in high dimensions.} For the policy gradient to be non-zero, the agent must encounter a non-negligible reward signal. Consider the residual space $\Omega = [0, M]^C$ (all possible per-constraint violation vectors), where $C$ is the constraint count. Also, consider the \emph{effective reward volume} $V_\epsilon \subset \Omega$ where $\mathcal{R}(\mathbf{r}) > \epsilon$ (the region where reward exceeds threshold $\epsilon$). We prove (Appendix~\ref{proof:reward_volume}) that as $C \to \infty$, $\mathrm{Vol}(V_\epsilon^G)/\mathrm{Vol}(\Omega) \to 0$ for global-norm rewards, while $\mathrm{Vol}(V_\epsilon^{\text{SAR}})/\mathrm{Vol}(\Omega) \to 1$ for SAR. While pretrained LLMs are not truly random, base model performance on GCS remains weak (Table~\ref{tab:gcs_merged}: Qwen3-8B achieves only 0.18 Hard SR), and early RL rollouts still sample broadly across $\Omega$. The probability of obtaining useful gradient signal in this regime decays exponentially for global norms and converges to one for SAR.

\paragraph{Gradient sensitivity.} Learning efficiency also depends on whether a local improvement in constraint satisfaction produces a measurable change in the reward, captured by $\|\nabla_{r_i}\mathcal{R}\|$. We show (Appendix~\ref{app:proof_1}) that the global-norm gradient $\|\nabla_{r_i}\mathcal{R}_G\| = |\phi'(L)| \cdot (r_i/L)^{p-1}$ is structurally coupled to the total error $L = \|\mathbf{r}\|_p$, whereas the SAR gradient $\|\nabla_{r_i}\mathcal{R}_{\text{SAR}}\| = |\phi'(r_i)|$ is strictly local. When the global configuration remains disordered ($L \gg r_i$), whether due to a single divergent outlier or aggregate error from a worsening subset, the global-norm signal for the successful constraint $i$ vanishes. A subset of worsening residuals can ``veto'' reinforcement for constraints that are actually improving. On the other hand, the SAR signal remains non-zero and independent of $L$, allowing valid sub-solutions to be protected from global noise.

\paragraph{Why compare SAR with global-norm rewards?} Mean Squared Error (MSE) and Sum of Squared Errors (SSE) are the native objectives of most mainstream geometry constraint solvers (SolveSpace, FreeCAD, GeoSolver), making $\exp(-\text{SSE}/T_{\text{mse}})$ the first dense reward a practitioner would try. In our experiments we set $T_{\text{mse}} = 10$, chosen because it maximizes reward spread on partially correct solutions. It gives MSE the best possible discrimination before comparing it with SAR. Despite this favorable tuning, MSE still collapses $60\%$ of partially correct solutions to near-zero (Appendix~\ref{app:gradient_analysis}), making the comparison conservative rather than a strawman. The standard sparse RLVR reward $\mathbb{I}[\|\mathbf{r}\|_2^2 < \epsilon]$ (1 for a fully correct solution, 0 otherwise) is itself a special case: as $T_{\text{mse}} \to 0$, the kernel $\exp(-\|\mathbf{r}\|_2^2/T_{\text{mse}})$ collapses to the indicator $\mathbb{I}[\|\mathbf{r}\|_2^2 = 0]$. 

\subsection{Reward function for GCS}
\label{subsec:gcs_reward}

A pure SAR reward is robust to outlier gradient masking, but two domain-specific failure modes prevent it from being sufficient on its own.

\paragraph{Reward plateau.} The dense SAR component rewards partial constraint satisfaction monotonically: a configuration that satisfies 14 of 16 constraints receives high reward even though it is, by GCS standards, wrong. The agent has little incentive to push the final residuals through the strict $\|\mathbf{r}\|_2^2 < 10^{-3}$ threshold required for a valid solution. A sparse outcome bonus, activated only when all constraints are simultaneously satisfied, supplies that force.

\paragraph{Geometric degeneracy.} Geometric constraints can be satisfied by trivial configurations, for instance, by collapsing all points to a single coordinate or placing every line on the same axis. These solutions can have no residual but are useless as geometric constructions. A degeneracy penalty, deducting reward proportional to the number of detected degenerate substructures, prunes them. 

\paragraph{Composite reward.} Combining the three components, we use:
\begin{equation}\label{eq:main_reward}
    \mathcal{R} = \underbrace{\frac{w}{C} \sum_{i=1}^{C} e^{-r_i/T}}_{\text{Dense Shaping (SAR)}} + \underbrace{\mathbb{I}_{\text{suc}} \cdot R_{\text{bonus}}}_{\text{Sparse Bonus}} - \underbrace{\min(4, C_{\text{deg}})}_{\text{Degeneracy Penalty}}
\end{equation}
where $\mathbb{I}_{\text{suc}}$ activates when $\|\mathbf{r}\|_2^2 < 10^{-3}$, $C_{\text{deg}}$ counts detected degenerate substructures (capped at 4), and we set $w = 6.0$, $T = 0.1$, and $R_{\text{bonus}} = 4.0$ to bound the reward in $[-4, 10]$ and create a sharp landscape that demands high precision. Hybrid dense and sparse structures have improved performance in general mathematical reasoning \citep{hybridreinforcementwhen}. Our design adapts this pattern to the precision-critical GCS setting. Appendix~\ref{app:engine} illustrates how the reward decreases as constructions deviate from valid geometric configurations (Figure~\ref{fig:reward_grid}).

\section{The PyGeoX-RL Framework}
\vspace{-0.5em}

In this section, we present PyGeoX-RL, an RL environment for teaching LLMs geometric precision.


\subsection{PyGeoX: A Programmable Geometric Constraint Environment}
\label{sec:pygeox}

GCS problems are typically under-constrained and possess an infinite solution space. As seen in Figure \ref{fig:PyGeoX}, the construction remains valid regardless of translation, rotation or scaling, meaning a direct coordinate-wise comparison with a ground-truth diagram is impossible. To align a LLM with geometric laws, we require an environment capable of rigorously quantifying ``correctness" while maintaining the flexibility to express complex design intent. We introduce PyGeoX, a lightweight, object-oriented Python framework for geometric constraint solving. Unlike traditional CAD tools that rely on manual pointer-clicking, PyGeoX is designed as a fully programmable DSL, enabling humans and LLMs to define geometry declaratively.





\begin{figure}[htp!]
    \centering
    \includegraphics[width=\textwidth]{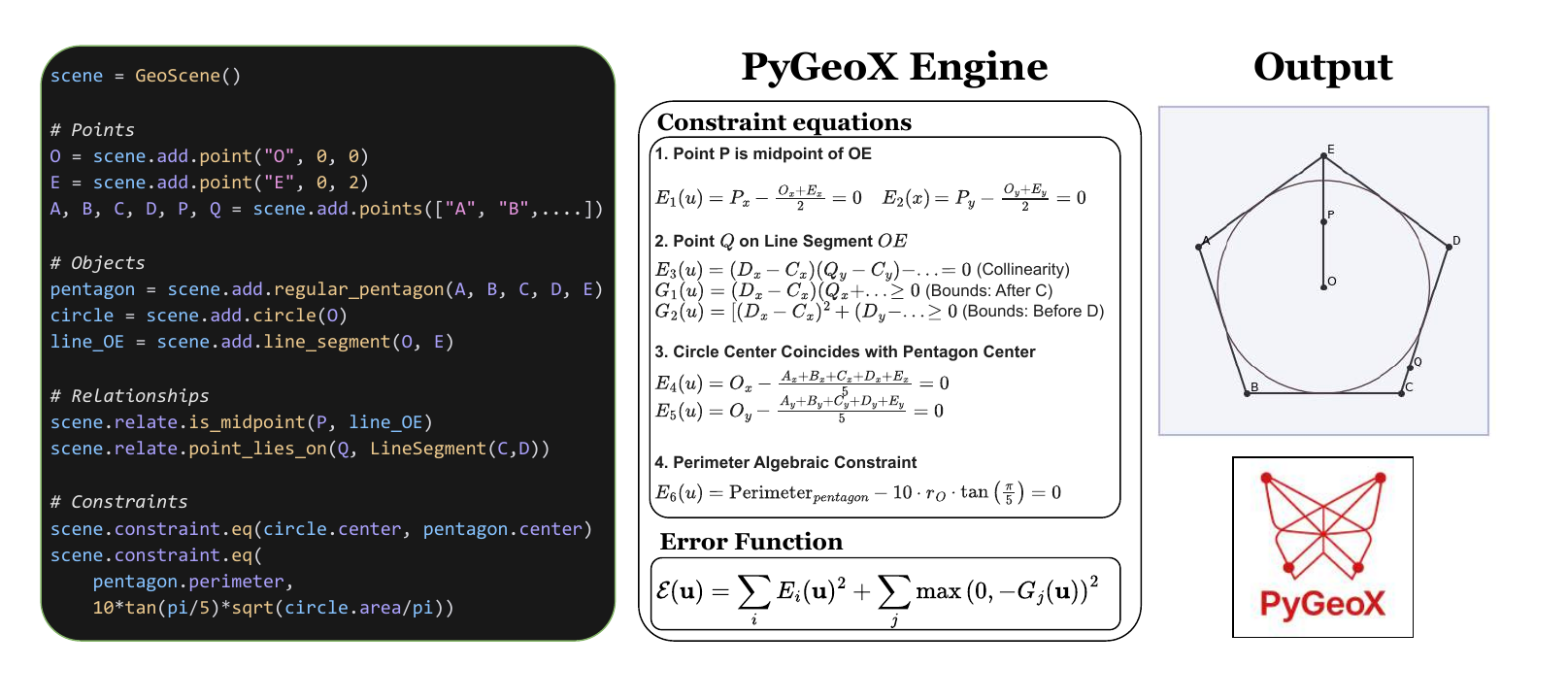} 
    \caption{The PyGeoX Symbolic Representation of a geometric figure (a pentagon ABCDE with incircle O and related line segment OE with midpoint P). The PyGeoX engine translates each geometric relationship into a system of symbolic constraints, $E_i = 0$ (equality) or $G_i > 0$ (inequality). These constraints are aggregated into a single error function $\mathcal{E}(\mathbf{u})$, where the input vector $\mathbf{u}$ comprises all geometric variables.} 
    \label{fig:PyGeoX}
\end{figure}
\vspace{-0.5em}

To overcome the limitations of the discrete and opaque engines discussed in Section~\ref{sec:related_work}, PyGeoX unifies symbolic expression with numeric feedback by implementing a three-stage Declarative-to-Symbolic-to-Differentiable Pipeline. As illustrated in Figure~\ref{fig:PyGeoX}, the engine first captures geometric intent through a high-level declarative DSL (left), which is lazily mapped into a stack of symbolic constraint equations via a SymPy backend (middle, top). These constraints are then aggregated into a single error function $\mathcal{E}(u)$ (middle, bottom), which serves two distinct roles: during data generation, PyGeoX actively optimizes $\mathcal{E}(u)$ to compute valid coordinates and during RL training, the engine evaluates the per-constraint residuals $\mathbf{r} = [r_1, \ldots, r_C]$ extracted from $\mathcal{E}(u)$ and applies the composite reward function (Eq.~\ref{eq:main_reward}) to provide dense training signal. Crucially, PyGeoX handles both equalities and inequalities as residuals, enabling the reward to capture strict, non-strict, and not-equal constraints. The declarative interface supports 35 geometric objects and 38 relationships, with arbitrary algebraic constraints compiled via Numba's JIT for 10–50× speedups over Python. Further details on PyGeoX are given in Appendix~\ref{app:engine}.

\subsection{Data Generation for RL and PyGeoX-Bench}
\label{sec:data_gen}

\begin{figure}[htp!]
    \centering
    \includegraphics[width=\textwidth]{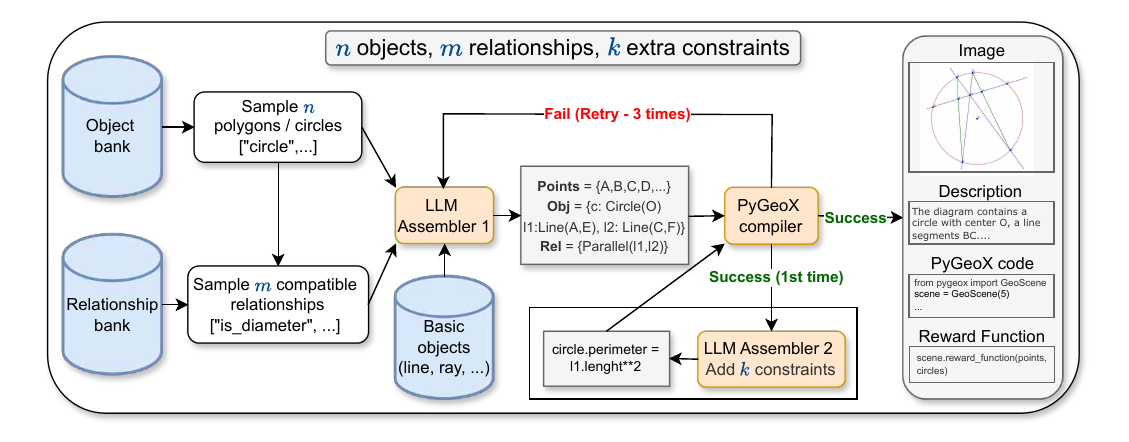} 
    \caption{Data generation pipeline.}
    \label{fig:data_training}
\end{figure}
\vspace{-0.25em}

As illustrated in Figure \ref{fig:data_training}, we utilized PyGeoX to construct a synthetic curriculum of 100k geometric problems. To ensure diversity and mitigate the repetitive outputs typical of unguided prompting, the pipeline seeds a Qwen3-30B-A3B model with $n$ objects, $m$ relationships and $k$ extra constraints randomly sampled from a weighted vocabulary of 35 object and 38 relationship types. The LLM expands these seeds into fully defined geometric specifications which are subsequently validated against the PyGeoX numerical solver and any configuration that fails to converge or exhibits degeneracy is automatically discarded. Each training sample is a 4-tuple: (1) natural language description of the diagram, (2) PyGeoX DSL code, (3) the compiled per-constraint reward function, and (4) a rendered image for visualization. These training samples are stratified into three difficulty tiers: Easy (single-polygon, $\sim$13 objects, 8 constraints), Medium (two-polygon, $\sim$15 objects, 10 constraints), and Hard (three-polygon, $\sim$23 objects, 16 constraints). Examples are found in Appendix~\ref{sec:appendix_examples}. This  generation pipeline was also employed to create our automated evaluation benchmark, PyGeoX-Bench, containing 100 problems for each difficulty tier.

\subsection{Training Methodology}
\label{subsec:training}

We evaluate SAR both as an RL reward and as an SFT sample-weighting scheme, using the same base model, Qwen3-8B. All agent outputs are executed within a sandboxed environment equipped with standard scientific libraries (\texttt{numpy}, \texttt{scipy}, \texttt{sympy}) under a 90-second timeout. Failure to produce a valid coordinate dictionary results in a reward of zero. The system prompt (Appendix~\ref{sec:appendix_prompt}) provides successful examples illustrating both numerical optimization and constructive algebraic strategies.

\paragraph{Supervised fine-tuning.} Standard SFT protocols filter data for strict correctness, discarding any solution that fails the final check. We hypothesized that this binary filtering throws away valuable logic contained in partially correct reasoning traces, where the model gets most constraints right but fails on a few. We generated 15{,}666 reasoning traces on medium-difficulty problems with a Qwen3-32B teacher and trained a Qwen3-8B student using the weighted SFT objective
\begin{equation*}
\mathcal{L} = -\sum_i \tilde{\mathcal{R}}_i \sum_j \log p(y_{i,j} \mid y_{i,<j}, \mathbf{x}_i),
\end{equation*}
where each sample is weighted by its normalized reward $\tilde{\mathcal{R}}_i \in [0, 1]$. This retains partial-credit logic from non-perfect solutions instead of discarding it. We adjust the number of epochs across variants so each model sees the same total token volume, and select the best checkpoint by validation loss. We use LoRA ($r{=}8$, $\alpha{=}32$, dropout $0.05$) with the \texttt{ms-swift} \citep{mswift} framework, learning rate $1 \times 10^{-5}$, cosine decay with 10\% warmup, and maximum sequence length $18{,}196$ tokens. 

\paragraph{Reinforcement learning.} We train each reward variant with Group Relative Policy Optimization (GRPO) via \texttt{OpenRLHF} \citep{openrlhf} on 10k medium-difficulty problems drawn from the procedural corpus, generating $G = 8$ rollouts per problem with a maximum generation length of $8{,}192$ tokens. Crucially, all RL runs are cold-started from the base Qwen3-8B model, not from any SFT checkpoint; this isolates the effect of reward design from any SFT initialization bias. To further isolate the impact of reward design, we keep optimization controlled across variants: learning rate $5 \times 10^{-6}$ with cosine decay, KL coefficient $0.01$, asymmetric clipping $\epsilon \in [0.2, 0.3]$, and discount factor $\gamma = 1.0$. Batch composition and the remaining hyperparameters are in Appendix~\ref{app:reproducibility}.

Table~\ref{tab:gcs_merged} reports the full five-way reward ablation across both settings.



\begin{table*}[t]
\centering
\caption{Five-way reward ablation on Qwen3-8B across SFT and RL. \textbf{SR} = solving rate (all constraints satisfied, $\|\mathbf{r}\|_2^2 < 10^{-3}$). \emph{S+D} denotes the sparse success bonus and degeneracy penalty (Eq.~\ref{eq:main_reward}); \emph{SAR+S+D} is the composite reward proposed in this paper. }
\label{tab:gcs_merged}
\resizebox{1\textwidth}{!}{
\begin{tabular}{l l c c c c}
\toprule
\textbf{Model} & \textbf{Reward} & \textbf{PyGeoX (E)} & \textbf{PyGeoX (M)} & \textbf{PyGeoX (H)} & \textbf{PyGeoX-Wild} \\
\midrule
Qwen3-8B (Base)        & none         & 0.55 & 0.37 & 0.18 & 0.57 \\
\midrule
Qwen3-8B-SFT           & SAR          & \textbf{0.59} & \textbf{0.55} & 0.23 & 0.56 \\
Qwen3-8B-SFT           & MSE          & 0.48 & 0.26 & 0.22 & 0.53 \\
Qwen3-8B-SFT           & Sparse       & 0.51 & 0.42 & 0.23 & 0.57 \\
Qwen3-8B-SFT           & SAR+S+D (Ours) & 0.55 & 0.47 & \textbf{0.32} & \textbf{0.65} \\
Qwen3-8B-SFT           & MSE+S+D      & 0.26 & 0.25 & 0.04 & 0.48 \\
\midrule
Qwen3-8B-RL            & SAR          & 0.46 & 0.41 & 0.09 & 0.62 \\
Qwen3-8B-RL            & MSE          & 0.47 & 0.41 & 0.10 & 0.55 \\
Qwen3-8B-RL            & Sparse       & \textbf{0.63} & \textbf{0.51} & 0.35 & 0.59 \\
Qwen3-8B-RL            & SAR+S+D (Ours) & 0.62 & 0.50 & \textbf{0.41} & \textbf{0.66} \\
Qwen3-8B-RL            & MSE+S+D      & 0.59 & 0.47 & 0.18 & 0.60 \\
\bottomrule
\end{tabular}
}
\end{table*}

\section{Experimental results}
\label{sec:experiments}

This section presents our evaluation protocol, comparative ablations, and analysis.

\subsection{Benchmarks and metrics}
\label{subsec:bench_metrics}

\textbf{PyGeoX-Bench} comprises 300 problems (100 each at Easy, Medium, Hard) drawn from the same procedural pipeline as training data but held-out for evaluation. \textbf{PyGeoX-Wild} is an out-of-distribution (OOD) diagnostic of 86 problems adapted from a published middle-school geometry benchmark (citation withheld for anonymity, included with code release). Unlike PyGeoX-Bench, PyGeoX-Wild uses human-authored natural prose rather than templated descriptions, invokes constraint types absent from the training distribution, and requires chaining geometric identities in unseen combinations. With 86 problems, PyGeoX-Wild is sized to detect whether performance gaps persist OOD, not to claim broad geometry reasoning. Performance is quantified by the one-shot \textit{solving rate} (SR), the fraction of problems satisfying $\|\mathbf{r}\|_2^2 < 10^{-3}$.

\subsection{Comparative analysis}
\label{subsec:comparative}

Both RL and SFT improve over the base model across most reward variants, confirming that GCS is learnable through policy optimization.

\paragraph{Headline finding: SAR beats the field-standard sparse RLVR baseline.}  The most direct comparison for our reward-design contribution is against \emph{sparse} rewards, since the binary success indicator $\mathbb{I}[\|\mathbf{r}\|_2^2 < \epsilon]$ is the standard RLVR signal. Adding our saturating per-constraint dense signal on top of this sparse outcome (SAR+S+D) substantially outperforms sparse alone on the hard tier (RL: $0.41$ vs.\ $0.35$; SFT: $0.32$ vs.\ $0.23$) and on PyGeoX-Wild (RL: $0.66$ vs.\ $0.59$; SFT: $0.65$ vs.\ $0.57$). The sparse outcome continues to act as a hard correctness gate and SAR adds the partial-progress signal that allows learning on problems the sparse indicator alone reports as failures.

\paragraph{The composite design is essential.} Pure dense rewards in isolation underperform: SAR alone and MSE alone reach only $0.09$--$0.10$ Hard SR under RL, well below sparse alone ($0.35$). Adding the sparse outcome bonus to either dense signal raises Hard SR to $0.18$--$0.41$. The dense component on its own does not converge to exact solutions. 

\paragraph{Secondary finding: MSE corroborates the gradient-masking analysis.} Among dense components, SAR wins decisively over the MSE reward ($\exp(-\text{SSE}/10)$): SAR+S+D vs.\ MSE+S+D yields $0.32$ vs.\ $0.04$ in SFT (an $8\times$ gap on Hard) and $0.41$ vs.\ $0.18$ in RL ($2.3\times$). Strikingly, MSE+S+D underperforms Sparse alone on the hard tier ($0.18$ vs.\ $0.35$ in RL), in line with the gradient-masking analysis: replacing SAR's per-constraint kernel with the standard global-norm objective actively harms training. An empirical analysis of 3{,}893 partially correct solutions (some but not all constraints satisfied) corroborates this: $97\%$ of SAR rewards fall in the informative range $[0.1, 0.9]$, whereas $60\%$ of MSE rewards collapse to near-zero, leaving GRPO unable to differentiate better partial solutions from worse ones (Appendix~\ref{app:gradient_analysis}).


\paragraph{Cross-distribution stability.} The same ranking emerges on PyGeoX-Bench (300 problems) and on PyGeoX-Wild (86 problems with entirely different problem sources, language, and constraint combinations). Two findings are robust across both: (i) the composite dense+sparse design is essential and SAR+S+D leads in both SFT (0.65) and RL (0.66) on PyGeoX-Wild; (ii) MSE+S+D actively harms learning, underperforming Sparse alone in SFT (0.48 vs.\ 0.57). The consistency suggests the ranking is not an artifact of the training distribution.

\paragraph{Frontier models.} For context, we compare our 8B model 
against DeepSeek-V3.2 and several proprietary frontier systems 
evaluated zero-shot. On the Hard tier, our model outperforms three 
of the four, with the strongest baseline reaching $0.51$ SR. Full 
numbers are reported in Appendix~\ref{app:frontier_context}. This 
comparison is included to contextualize task difficulty, not as a 
controlled baseline.

\paragraph{Token efficiency.} As shown in Figure~\ref{fig:reward_analysis}, the SAR+S+D-trained model averages $\sim$4{,}060 generated tokens per task, against $\sim$5{,}260 for the sparse baseline (a 22.8\% reduction). SAR appears to enable the model to find geometric solutions more directly, without the lengthy numerical-search traces the sparse-trained model produces.

\section{Discussion and Limitations}
\label{sec:discussion}

A natural question is whether the model actually internalized geometric law or merely memorized training examples. Three lines of evidence support internalization:
\begin{enumerate}[leftmargin=*, itemsep=2pt, topsep=2pt]
    \item \textbf{Constructive reasoning traces.} About $90\%$ of successful Hard-tier traces follow constructive ruler-and-compass-style derivations, reasoning step-by-step through geometric properties rather than emitting literal lookups or calling a black-box optimizer (Appendix~\ref{sec:appendix_examples}).
    \item \textbf{OOD transfer.} SAR+S+D obtains the highest PyGeoX-Wild SR in both SFT (0.65) and RL (0.66), despite three-way distribution shift: human-authored prose, unseen constraint types (angle bisectors, equal arc lengths), and novel multi-step reasoning chains.
    \item \textbf{Combinatorial infeasibility of memorization.} A Hard-tier problem draws 3 polygons from $\sim$15 subtypes, $\sim$23 objects, and 16 constraints from 38 relationship types. A conservative lower bound on distinct configurations exceeds $10^{17}$, dwarfing the $\sim$10k RL training problems by 13 orders of magnitude. Memorization cannot account for the observed performance.
\end{enumerate}

Our work has limitations that frame the scope of these claims. All RL and SFT results use Qwen3-8B as the base model. We attempted SFT on Qwen3-1.7B and Llama-3.1-8B-Instruct, but found the base model too weak at math and instruction following for this task, with near-zero base performance and unstable training across every reward configuration (Hard SR $\leq 0.02$). This suggests strong math and code base model capabilities are a prerequisite for GCS at this scale. A multi-scale study across more recent models such as Qwen3-4B, Qwen3-14B, Mistral-7B, and Gemma-4 would strengthen the paper but exceeded our compute budget. The current engine also targets 2D static geometry, while kinematic synthesis and 3D CAD extension require additional symbolic primitives.

\section{Conclusion}
\label{sec:conclusion}

We presented PyGeoX-RL, a framework that formulates end-to-end Geometric Constraint Solving as an LLM alignment task: natural language to exact metric coordinates verified by per-constraint solver residuals. By having the LLM directly emit coordinates rather than delegating to a DSL or external solver API, our approach enables open-ended geometric synthesis unconstrained by predefined vocabularies. We released the PyGeoX engine and DSL, the PyGeoX-Bench evaluation suite, the PyGeoX-Wild OOD diagnostic, and the data-generation pipeline. Within this infrastructure we identified Outlier Gradient Masking and showed that SAR+S+D outperforms field-standard rewards on hard GCS problems. Beyond GCS, SAR applies to any RLVR setting where a solver returns a multi-dimensional residual over many constraints: physical simulation, scene layout or robotic manipulation. We release the engine, data, and benchmark to accelerate further research into solver-grounded alignment for precision-critical domains (see ``Software and Data'').

\section*{Reproducibility and Software}

We release the full research stack: PyGeoX engine and DSL, training and evaluation pipeline, PyGeoX-Bench (300 problems), PyGeoX-Wild (86 problems), training data, and model checkpoints at \url{https://github.com/Huawei-AI4Math/PyGeoX}. The repository contains all configuration files (SFT and RL hyperparameters, training scripts, evaluation harness) needed to reproduce every row of Table~\ref{tab:gcs_merged}, main-text experimental settings are summarized in Section~\ref{sec:experiments} and consolidated alongside compute and seeds in Appendix~\ref{app:reproducibility}.
\bibliographystyle{plainnat}
\bibliography{paper}

\newpage
\appendix

\section{Theorem statements and proofs}

We restate the two formal results summarized informally in Section~\ref{subsec:sar} and provide full proofs.

\begin{theorem}[Vanishing vs.\ Concentrating Signal Volumes]
\label{thm:reward_volume}\label{proof:reward_volume}
Consider a residual space bounded by a hypercube $\Omega = [0, M]^C$, the global-norm reward $\mathcal{R}_G = \phi(\|\mathbf{r}\|_p)$, and the SAR reward $\mathcal{R}_{\text{SAR}} = \sum_i \phi(r_i)$, where $\phi : [0, \infty) \to [0, 1]$ is continuous, monotonically decreasing to zero with $\lim_{x\to\infty}\phi'(x)=0$. Define the effective reward volume $V_\epsilon = \{\mathbf{r} \in \Omega : \mathcal{R}(\mathbf{r}) > \epsilon\}$. Then as the constraint count $C \to \infty$,
\[
\lim_{C \to \infty} \frac{\mathrm{Vol}(V_\epsilon^G)}{\mathrm{Vol}(\Omega)} = 0 \quad\text{and}\quad \lim_{C \to \infty} \frac{\mathrm{Vol}(V_\epsilon^{\text{SAR}})}{\mathrm{Vol}(\Omega)} = 1.
\]
\end{theorem}
\begin{proof}
\textbf{Part 1: Global Norm (Vanishing Volume)}
For a Global Norm reward $\mathcal{R}_G = \phi(\|\mathbf{r}\|_p)$ to exceed $\epsilon$, we require $\|\mathbf{r}\|_p < \phi^{-1}(\epsilon) := R_{\max}$. The effective reward region $V_\epsilon^G \subseteq \Omega = [0, M]^C$ is the set of all residual vectors satisfying this constraint. This region is bounded by a $C$-dimensional $\ell_p$-ball of radius $R_{\max}$. The volume ratio can be bounded as:
\[
\frac{\text{Vol}(V_\epsilon^G)}{\text{Vol}(\Omega)} \leq \frac{\text{Vol}(\text{$C$-ball of radius } R_{\max})}{M^C} \propto \left(\frac{R_{\max}}{M}\right)^C
\]

where the proportionality constant depends on the dimension $C$ and norm type $p$, but crucially, $R_{\max}$ is a constant determined solely by $\epsilon$ and the kernel function $\phi$. Since $0 < R_{\max} < M$ (assuming the effective reward region is strictly contained in the domain), we have:
\[
\lim_{C \to \infty} \left(\frac{R_{\max}}{M}\right)^C = 0
\]

\textbf{Part 2: SAR (Concentrating Volume)}
For SAR, the condition is $\sum_{i=1}^C \phi(r_i) > \epsilon$.
Let $X_i = \phi(r_i)$ be a random variable corresponding to the kernel value of a randomly sampled point in $\Omega$. Since $\phi$ is strictly decreasing and bounded $[0,1]$, and the domain is $[0,M]$, the expected value $\mu = \mathbb{E}[X_i]$ is strictly positive.
By the Law of Large Numbers, the sum $\sum_i X_i$ concentrates around $C\mu$.
For any fixed threshold $\epsilon$, the condition $\epsilon < C\mu$ is eventually satisfied as $C \to \infty$ since $C\mu$ grows without bound while $\epsilon$ remains fixed. Formally, using Hoeffding's inequality for $C$ bounded i.i.d.\ variables in $[0,1]$, the measure of the complement set (where reward $\le \epsilon$) decays exponentially:
$$ \frac{\text{Vol}(\Omega \setminus V_\epsilon^{SAR})}{\text{Vol}(\Omega)} \le \exp\left(-\frac{2(C\mu - \epsilon)^2}{C}\right) $$
Thus, the relative volume of the reward region approaches 1.
\end{proof}

\begin{corollary}[Initialization Success]
Under random initialization on $\Omega$, the probability of obtaining a meaningful learning signal ($\mathcal{R} > \epsilon$) decays exponentially to $0$ for global-norm rewards and converges to $1$ for SAR.
\end{corollary}

\begin{theorem}[Robustness to Global Error]
\label{thm:gradient_decoupling}\label{app:proof_1}
Let $\mathbf{r} \in \mathbb{R}^C$ be the residual vector and $L = \|\mathbf{r}\|_p$ the global $p$-norm ($p \ge 1$). The gradient magnitudes of the global-norm and SAR rewards with respect to a single residual $r_i$ are
\[
\|\nabla_{r_i}\mathcal{R}_G\| = |\phi'(L)| \cdot (r_i/L)^{p-1}, \qquad \|\nabla_{r_i}\mathcal{R}_{\text{SAR}}\| = |\phi'(r_i)|.
\]
The global-norm gradient is structurally coupled to the global error $L$, whereas the SAR gradient is strictly local.
\end{theorem}
\begin{proof}
\textbf{Part 1 (Global):} Let $L(\mathbf{r}) = (\sum_k r_k^p)^{1/p}$. By the Chain Rule, $\frac{\partial \mathcal{R}_G}{\partial r_i} = f'(L) \cdot \frac{\partial L}{\partial r_i}$.
The derivative of the $p$-norm is $\frac{\partial L}{\partial r_i} = \frac{1}{p}(\sum r_k^p)^{\frac{1}{p}-1} \cdot p r_i^{p-1} = L^{1-p} r_i^{p-1} = (\frac{r_i}{L})^{p-1}$.
Substituting this back yields $\frac{\partial \mathcal{R}_G}{\partial r_i} = f'(L) (\frac{r_i}{L})^{p-1}$.

\textbf{Part 2 (Additive):} We apply the sum rule to $\mathcal{R}_{SAR}$. Since $\mathcal{R}_A = \phi(r_i) + \sum_{j \neq i} \phi(r_j)$, the derivative of the sum of other terms with respect to $r_i$ is zero. The gradient is simply the derivative of the local term: $\phi'(r_i)$.
\end{proof}

\begin{corollary}[Suppression of Partial Solutions]
When the global configuration is disordered ($L \gg r_i$), whether due to a single divergent outlier ($r_j \to \infty$) or aggregate error from a worsening subset, the global-norm signal for the successful constraint $i$ vanishes, $\|\nabla_{r_i}\mathcal{R}_G\| \to 0$, while the SAR signal $\|\nabla_{r_i}\mathcal{R}_{\text{SAR}}\| = |\phi'(r_i)|$ remains non-zero and independent of $L$.
\end{corollary}

\section{PyGeoX Engine}
\label{app:engine}

The PyGeoX engine is architected as a Symbolic-to-Differentiable pipeline that lowers high-level geometric intent into an optimized numerical kernel. This architecture bridges the gap between discrete geometric logic and continuous optimization by maintaining a symbolic intermediate representation throughout the translation process.

\subsection{Object Model and API Structure}
PyGeoX represents a geometric scene as a graph $G = (V, E)$, where vertices $V$ are geometric objects and edges $E$ are constraints. The library is organized into three distinct namespaces that facilitate this construction:

\begin{enumerate}
    \item \textbf{Object Instantiation (\texttt{scene.add}):} This namespace handles the creation of geometric primitives. The library supports over 30 distinct object types, ranging from fundamental primitives (Points, Rays, Arcs) to a rich hierarchy of polygons (e.g., \texttt{RightTrapezoid}, \texttt{RegularOctagon}). To ensure logical consistency, PyGeoX utilizes a strict type hierarchy where specific properties are inherited automatically; for example, a \texttt{RegularPentagon} inherits from \texttt{RegularPolygon}, which in turn inherits from \texttt{Polygon}. This structure allows specific properties (e.g., apothem, internal angle) to be exposed automatically while preventing invalid operations on incompatible types.
    
    \item \textbf{Geometric Relationships (\texttt{scene.relate}):} These methods establish topological dependencies between objects. PyGeoX provides over 30 geometric relationships, covering incidence (e.g., \texttt{point\_lies\_on}, \texttt{collinear}), construction (e.g., \texttt{is\_circumcircle}, \texttt{is\_orthocenter}), and rigid body transformations (e.g., \texttt{rotation\_around\_point}, \texttt{mirror\_across\_line}). Internally, these high-level semantic relationships are decomposed into their constituent algebraic equations.
    
    \item \textbf{Property Constraints (\texttt{scene.constraint}):} This namespace provides the interface for defining arbitrary algebraic constraints (equalities and inequalities) on object properties.
\end{enumerate}

A distinguishing feature of PyGeoX is the ability to impose equality or inequality constraints on \textit{any} derived scalar property of a geometric object. While prior geometric reasoning systems such as AlphaGeometry \citep{trinh2024alphageometry} or FormalGeo \citep{zhang2023formalgeo} typically rely on a fixed set of geometric predicates (e.g., static relationships like \texttt{same\_length}), they lack the expressivity to handle arbitrary algebraic relationships between different object properties, for instance, to express that ``line 1 length is larger than the circle 1 area".

In contrast, PyGeoX allows the definition of constraints that bridge different geometric domains. For instance, a user can  enforce a relationship between two entirely different shapes, such as requiring the perimeter of a pentagon to equal the area of a circle (\texttt{scene.constraint.eq(pentagon.perimeter, circle.area)}). This flexibility drastically expands the scope of solvable problems beyond standard textbook Euclidean problems. 

\subsection{Symbolic Translation and Solving}
As illustrated in Figure \ref{fig:PyGeoX}, PyGeoX employs a lazy evaluation strategy, mapping declarative relationships to a symbolic buffer of SymPy equations rather than computing them instantly. This global view supports a diverse set of algebraic conditions, including strict ($S_k > 0$) and non-strict ($G_j \geq 0$) inequalities, enabling the expression of complex topological constraints.

Prior to evaluation, the symbolic stack undergoes automated simplification. The engine prunes redundant constraints and re-parameterizes primitives (reducing free variables by 20--40\%) while performing radical elimination to transform non-linear square roots into smooth polynomial forms. These optimized expressions are aggregated into a differentiable loss $\mathcal{E}(\mathbf{u})$ and compiled via Numba's JIT compiler.

Crucially, this compiled loss serves two distinct roles depending on the pipeline phase:
\begin{itemize}
    \item Data Generation: The engine actively guarantees geometric validity of the randomly generated diagrams. We optimize $\mathbf{u}$ against $\mathcal{E}(\mathbf{u})$ using Basin-hopping; if the solver fails to converge to near-zero error, the diagram is flagged as potentially contradictory or over-constrained and is automatically discarded.
    \item RL Training: The engine shifts to a passive role. The problem is extracted as a pair consisting of the natural language description of the diagram and the executable residuals $\{r_i\}$. The RL agent bears full responsibility for GCS, while the engine merely evaluates these residuals to compute the reward signal.
\end{itemize}

\subsection{Objective Function Compilation}

To provide the low-latency feedback required for Reinforcement Learning (RL), the symbolic residual stack is lowered into a scalar error function $\mathcal{E}(\mathbf{u})$ via Numba's Just-In-Time (JIT) compilation. This generates optimized machine code that utilizes efficient NumPy array operations, bypassing Python interpreter overhead and yielding 10--50$\times$ speedups. The objective function aggregates four distinct constraint classes into a single scalar loss:

\begin{equation}
\mathcal{E}(\mathbf{u}) = \alpha \sum_i E_i(\mathbf{u})^2 + \beta \sum_j \max(0, -G_j(\mathbf{u}))^2 + \beta \sum_k \max(0, \epsilon - S_k(\mathbf{u}))^2 + \gamma \sum_\ell \mathds{1}_{|N_\ell(\mathbf{u})| < \epsilon}
\label{eq:objective_function}
\end{equation}

where $E$, $G$, $S$, and $N$ represent equality ($=0$), non-strict inequality ($\geq 0$), strict inequality ($> 0$), and not-equal ($\neq 0$) constraints, respectively. The default weights are $\alpha = \beta = \gamma = 1.0$, with strictness tolerance $\epsilon = 10^{-4}$.

\subsection{Hybrid Global Optimization and Degeneracy Handling}

The resulting differentiable landscape is explored using a hybrid global optimization strategy. PyGeoX defaults to Basin-hopping with 1000 iterations, which alternates between stochastic coordinate perturbations (step size $0.5$, temperature $T=2$) and local L-BFGS-B refinement to escape non-convex local minima. Alternative methods include direct L-BFGS-B optimization (for well-conditioned problems) and dual annealing (for highly nonlinear systems).

To ensure structural validity and prevent coordinate collapse, where points converge to a single location to trivially satisfy distance constraints,the solver applies an optional separation penalty:
\begin{equation}
\mathcal{L}_{sep} = \delta \cdot \sum_{i \neq j} \max(0, d_{min}^2 - \|p_i - p_j\|^2)
\end{equation}
where $d_{min}$ is automatically computed as $1/200$ of the domain span (e.g., $d_{min} = 1.0$ for a $[-50,50]^2$ workspace) and $\delta = 1.0$ controls penalty strength. This ensures the engine provides the continuous, non-degenerate diagnostic signal required for model alignment.

\begin{figure}[h]
    \centering
    \renewcommand{\arraystretch}{1.2}
    \begin{tabular}{ccc}
        \includegraphics[width=0.3\columnwidth]{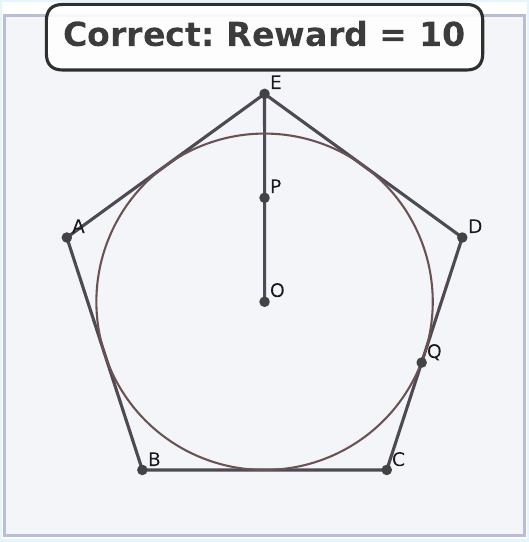} &
        \includegraphics[width=0.3\columnwidth]{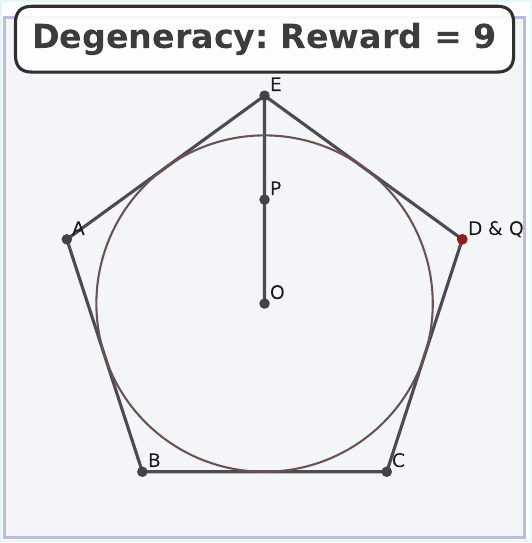} &
        \includegraphics[width=0.3\columnwidth]{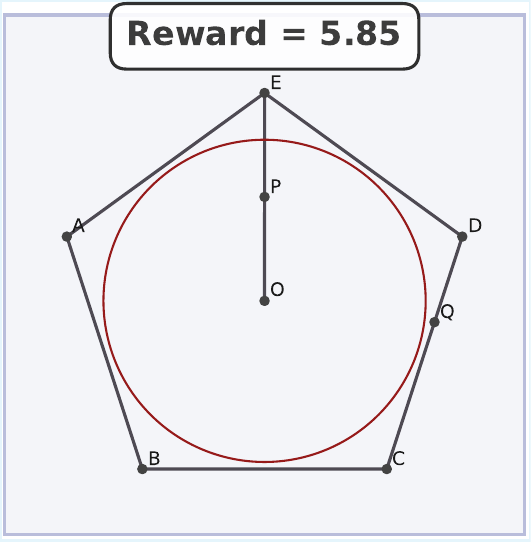} \\
        \includegraphics[width=0.3\columnwidth]{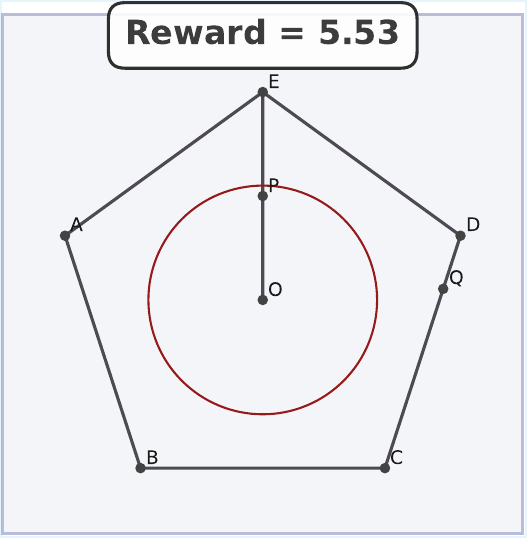} &
        \includegraphics[width=0.3\columnwidth]{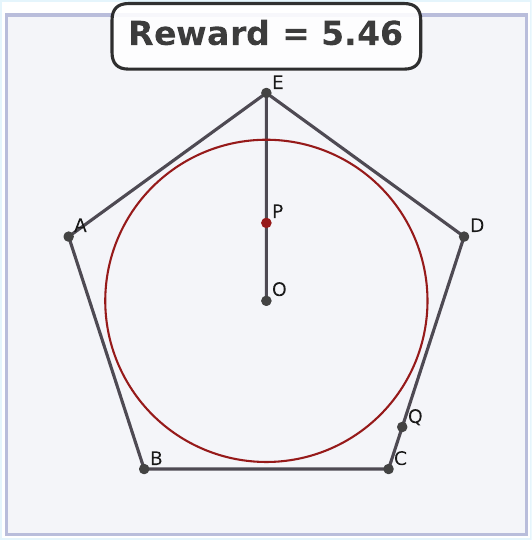} &
        \includegraphics[width=0.3\columnwidth]{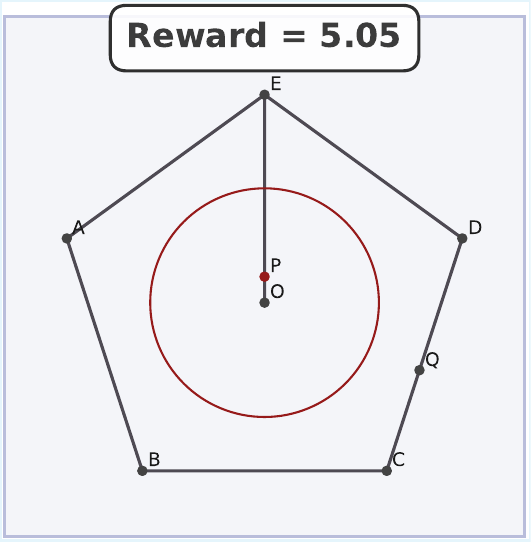} \\
    \end{tabular}
    \caption{Visual progression of the reward landscape for the diagram in Figure~\ref{fig:PyGeoX}. The top row shows all constraints satisfied; the middle row shows one broken constraint (the circle is not an incircle); and the bottom row shows two broken constraints (additionally, $P$ is not the midpoint).}
    \label{fig:reward_grid}
\end{figure}

\subsection{Scope and comparison with other tools}
\label{app:scope}

Tables \ref{tab:objects}, \ref{tab:relationships-compact}, \ref{tab:constraints} and \ref{tab:object-properties-compact} provide a comprehensive catalog of the PyGeoX DSL. This vocabulary encompasses a wide array of geometric primitives and relational predicates that allow for the declarative specification of scene topology. Additionally, the constraint operators listed in Table \ref{tab:constraints} enable fine-grained control over metric properties, allowing users to impose precise algebraic or inequality-based conditions on the geometry.

\begin{table}[H]
\centering
\caption{PyGeoX Object Library Taxonomy}
\label{tab:objects}
\small
\setlength{\tabcolsep}{4pt} 
\begin{tabular}{@{}ll@{}}
\toprule
\textbf{Category} & \textbf{Object Types} \\
\midrule
\textbf{Basic Primitives} & Point, Line, LineSegment, Ray, Circle, MinorArc, MajorArc, Angle \\
\midrule
\textbf{Polygons} & Polygon, CyclicPolygon \\
\midrule
\textbf{Regular Polygons} & RegularPolygon, RegularPentagon, RegularHexagon, RegularHeptagon, RegularOctagon \\
\midrule
\textbf{Triangles} & Triangle, EquilateralTriangle, IsoscelesTriangle, ScaleneTriangle, \\
& RightTriangle, RightIsoscelesTriangle, AcuteTriangle, ObtuseTriangle \\
\midrule
\textbf{Quadrilaterals} & Quadrilateral, Parallelogram, Rectangle, Square, Rhombus, \\
& Trapezoid, IsoscelesTrapezoid, RightTrapezoid, Kite, Rhomboid \\
\bottomrule
\end{tabular}
\end{table}

\begin{table}[h]
\centering
\small
\caption{PyGeoX Relationships}
\label{tab:relationships-compact}
\begin{tabular}{@{}p{5cm}p{4cm}p{4cm}@{}}
\toprule
\textbf{Line Relationships} & \textbf{Incidence} & \textbf{Circle \& Arcs} \\
\midrule
\texttt{parallel} & \texttt{collinear} & \texttt{tangent\_to\_circle} \\
\texttt{perpendicular} & \texttt{point\_lies\_on} & \texttt{is\_chord} \\
\texttt{lines\_intersect\_at} & \texttt{points\_lie\_on} & \texttt{line\_intersects\_circle\_at} \\
\texttt{line\_extensions\_intersect\_at} & & \\
\midrule
\textbf{Special Lines/Points} & \textbf{Polygon-Specific} & \textbf{Angles} \\
\midrule
\texttt{perpendicular\_bisector\_at} & \texttt{is\_circumcircle} & \texttt{acute\_angle} \\
\texttt{angle\_bisector} & \texttt{is\_incircle} & \texttt{right\_angle} \\
\texttt{is\_midpoint} & \texttt{is\_orthocenter} & \texttt{obtuse\_angle} \\
\texttt{is\_radius} & \texttt{is\_centroid} & \\
\texttt{is\_diameter} & \texttt{is\_median} & \\
 & \texttt{is\_altitude} & \\
\midrule
\textbf{Containment} & \textbf{Congruence \& Similarity} & \textbf{Transformations} \\
\midrule
\texttt{inside} & \texttt{congruent} & \texttt{translation} \\
\texttt{outside} & \texttt{similar} & \texttt{scale} \\
 & & \texttt{rotation\_around\_point} \\
 & & \texttt{mirror\_across\_line} \\
\bottomrule
\end{tabular}
\end{table}

\begin{table}[h]
\centering
\small
\caption{PyGeoX Constraint Types}
\label{tab:constraints}
\begin{tabular}{@{}llp{7cm}@{}}
\toprule
\textbf{Category} & \textbf{Constraint} & \textbf{Mathematical Form} \\
\midrule
\textbf{Equality} 
& \texttt{eq(a, b)} & $a = b$ \\
\midrule
\multirow{5}{*}{\textbf{Inequality}} 
& \texttt{neq(a, b)} & $a \neq b$ (not-equal constraint) \\
& \texttt{gt(a, b)} & $a > b$ (strict greater-than) \\
& \texttt{lt(a, b)} & $a < b$ (strict less-than) \\
& \texttt{geq(a, b)} & $a \geq b$ (greater-than-or-equal) \\
& \texttt{leq(a, b)} & $a \leq b$ (less-than-or-equal) \\
\bottomrule
\end{tabular}
\vspace{0.3cm}

\noindent\textbf{Note:} Arguments $a$ and $b$ can be any symbolic expression derived from object properties (e.g., coordinates, distances, angles, areas).
\end{table}

\begin{table}[h]
\centering
\small
\caption{PyGeoX Object Properties and Scalar Attributes}
\label{tab:object-properties-compact}
\small
\setlength{\tabcolsep}{5pt}
\begin{tabular}{@{}ll@{}}
\toprule
\textbf{Base Class} & \textbf{Common Properties and Derived Scalars} \\
\midrule
\textbf{Point} & $x, y$, \texttt{distance(P)}, \texttt{distance\_squared(P)}, \texttt{angle(A, B)} \\
\midrule
\textbf{Linear} & \texttt{slope}, \texttt{intercept}, \texttt{direction\_vector}, \texttt{equation} ($Ax+By+C=0$) \\
\textit{(Line, Segment, Ray)} & \textit{Segments only:} \texttt{length}, \texttt{midpoint} \\
\midrule
\textbf{Curvilinear} & \texttt{center}, \texttt{radius}, \texttt{area}, \texttt{perimeter} (circumference), \texttt{diameter} \\
\textit{(Circle, Arcs)} & \textit{Arcs only:} \texttt{central\_angle}, \texttt{length}, \texttt{midpoint}, \texttt{inscribed\_angle} \\
\midrule
\textbf{Polygonal} & \texttt{points}, \texttt{sides}, \texttt{num\_sides}, \texttt{area}, \texttt{perimeter}, \texttt{centroid}, \texttt{internal\_angles} \\
\textit{(General $n$-gon)} & \textit{Cyclic:} \texttt{circumcenter}, \texttt{circumradius} \\
\midrule
\textbf{Regular $n$-gon} & \texttt{center}, \texttt{circumradius}, \texttt{inradius} (apothem), \texttt{side\_length}, \\
& \texttt{internal\_angle}, \texttt{orientation} (rotation in radians) \\
\midrule
\textbf{Triangle} & \texttt{circumcenter}, \texttt{circumradius}, \texttt{incenter}, \texttt{inradius}, \texttt{orthocenter}, \\
& \texttt{medians}, \texttt{altitudes}, \texttt{midsegments}, \texttt{angle\_bisectors}, \texttt{heights} \\
\midrule
\textbf{Quadrilateral} & \texttt{midsegments}, \texttt{diagonals} \\
\textit{(Rect, Square, Rhombus)} & \textit{Specific solvers:} \texttt{incenter/radius} (Square/Rhombus), \texttt{circumcenter} (Rect) \\
\bottomrule
\end{tabular}
\end{table}

To contextualize the technical positioning of PyGeoX, we compare its capabilities against a spectrum of existing geometric and symbolic frameworks in Table~\ref{tab:system_comparison}, namely AlphaGeometry \citep{trinh2024alphageometry}, FormalGeo \citep{zhang2023formalgeo}, SolveSpace \citep{solvespace}, GeoGebra \citep{geogebra}, Sympy \citep{sympy}, Z3 (SMT) \citep{smt} and Geometry Expert \citep{geometry_expert}. The current landscape is generally bifurcated into two distinct paradigms. On one side are deductive reasoning engines such as \textit{AlphaGeometry} and \textit{FormalGeo}, which are primarily architected for discrete theorem proving. These systems excel at rigor but function as ``binary filters," offering sparse Valid/Invalid feedback, making them ill-suited for geometric synthesis tasks. They typically lack support for inequalities and continuous optimization, limiting their utility in synthesis tasks where finding \textit{a} valid configuration is the priority.

On the other side are numerical and interactive solvers utilized in CAD and education, such as \textit{SolveSpace} and \textit{GeoGebra}. While these tools handle metric constraints and visual rendering effectively, they often rely on ``black-box" numerical methods or manual user interaction. As shown in Table~\ref{tab:system_comparison}, they generally lack a differentiable backend, which prevents them from being controlled programmatically by optimization algorithms in a robust manner. Other symbolic tools like \textit{SymPy Geometry} or general-purpose SMT solvers like \textit{Z3} offer strong logical foundations but lack the granular feedback required to guide a solver toward valid geometric manifolds efficiently. As detailed in the comparison table, PyGeoX is unique in its ability to combine declarative input with a differentiable loss landscape, robust inequality handling, and automated degeneracy penalization—features that are typically fragmented across existing deductive and numerical frameworks.


\begin{table}[h]
\centering
\footnotesize
\setlength{\tabcolsep}{5pt}
\begin{tabular}{l c c c c c}
\toprule
\textbf{Feature} & \textbf{PyGeoX} & \textbf{AlphaGeometry} & \textbf{FormalGeo} & \textbf{SolveSpace} & \textbf{GeoGebra} \\
\midrule
Primary Goal    & Synthesis / Opt.   & Olympiad Proofs   & Theorem Proving   & CAD Modeling   & Education \\
Logic Type      & Neuro-Symbolic     & Neuro-Symbolic    & Symbolic (Formal) & Numeric        & Numeric/Sym \\
Solving Logic   & Global Opt.        & Deductive Search  & Logic Deduction   & Newton-Raphson & Dynamic \\
Differentiable  & \textbf{Yes}       & No                & No                & No             & No \\
Inequalities    & \textbf{Yes}       & No                & No                & Limited        & Partial \\
Degeneracy      & \textbf{Automated} & N/A               & Rules             & User-Managed   & N/A \\
Input Style     & Declarative        & NLP + Formal      & Formal Lang       & GUI / Script   & GUI \\
\bottomrule
\end{tabular}
\caption{Comparison of PyGeoX against geometry-specific frameworks: deductive theorem-proving systems (AlphaGeometry, FormalGeo) and numerical CAD/educational solvers (SolveSpace, GeoGebra). Comparisons against general-purpose symbolic frameworks are given in Table~\ref{tab:system_comparison_general}.}
\label{tab:system_comparison}
\end{table}

\begin{table}[h]
\centering
\footnotesize
\setlength{\tabcolsep}{6pt}
\begin{tabular}{l c c c c}
\toprule
\textbf{Feature} & \textbf{PyGeoX} & \textbf{SymPy} & \textbf{Z3 (SMT)} & \textbf{Geom. Exp.} \\
\midrule
Primary Goal    & Synthesis / Opt.   & Symbolic Math & Logic / Algebra & Symbolic Engine \\
Logic Type      & Neuro-Symbolic     & Symbolic      & Symbolic        & Symbolic \\
Solving Logic   & Global Opt.        & N/A           & SMT Solver      & Symbolic \\
Differentiable  & \textbf{Yes}       & No            & No              & Partial \\
Inequalities    & \textbf{Yes}       & Yes           & Full            & Yes \\
Degeneracy      & \textbf{Automated} & N/A           & N/A             & N/A \\
Input Style     & Declarative        & Imperative    & Formula         & GUI / Sym \\
\bottomrule
\end{tabular}
\caption{Comparison of PyGeoX against general-purpose symbolic frameworks (SymPy, Z3 SMT solver, Geometry Expert). Together with Table~\ref{tab:system_comparison}, this shows that PyGeoX is unique in combining a differentiable backend, native inequality support, and automated degeneracy handling.}
\label{tab:system_comparison_general}
\end{table}

\clearpage
\section{RL training details and experimental results}

\subsection{System Prompt}
\label{sec:appendix_prompt}

We provide the full system prompt used to align the LLM for geometric reasoning and Python code generation during the GRPO training. This system prompt is used for all single-shot model evaluations shown in Section \ref{sec:experiments} and Table \ref{tab:gcs_merged} and was also used during the RL training.

\begin{promptbox}{System Prompt: The Geometry Architect}
\small
You are an \textbf{Expert Computational Geometer and Python Algorithmist}. Your objective is to solve geometry problems by generating high-quality reasoning traces and precision Python code. You operate in two primary modes: \textbf{Constructive Algebra} and \textbf{Numerical Optimization}.

{[...]} 

\vspace{3mm}
\textbf{STRATEGY EXAMPLES}

\vspace{2mm}
\textbf{Example 1: Numerical Optimization (Basinhopping)} \\
\textbf{Problem:} Construct a right-angled triangle with side lengths $3$, $4$, and hypotenuse $5$. Find the coordinates of a point $P$ located strictly inside the triangle that is equidistant from all three sides (the incenter). Provide the coordinates of $P$ and the radius of the circle centered at $P$ that is tangent to the sides.

\vspace{2mm}
\textbf{\texttt{<think>}}
\textbf{Phase A: Construction Strategy}
\begin{enumerate}
    \item \textbf{Analyze:} The goal is to construct a classic \textbf{3-4-5 Right Triangle} and locate its \textbf{Incenter} P. The incenter is the unique point equidistant from all sides of the triangle.
    \item \textbf{Anchoring:} To simplify the search space and prevent the shape from spinning or drifting, I will anchor vertex B at the origin (0,0). I will constrain C to lie on the positive X-axis and A to lie on the positive Y-axis. This reduces the optimization from 8 variables down to 4.
    \item \textbf{Derive Constraints:}
    \begin{itemize}
        \item \textbf{Geometry:} The distances AB, BC, and AC must be 3, 4, and 5 respectively.
        \item \textbf{Incenter Property:} The distance from P to line AB (the Y-axis), line BC (the X-axis), and the hypotenuse AC must be equal.
        \item \textbf{Containment:} P must strictly reside within the triangle. I will use the cross-product winding method to penalize the solver if P drifts outside the boundary.
    \end{itemize}
\end{enumerate}
\textbf{Phase B: Implementation Mapping}
\begin{enumerate}
    \setcounter{enumi}{3}
    \item \textbf{Tooling:} I will use \texttt{scipy.optimize.basinhopping} for global optimization and \texttt{numba.jit} to ensure the loss function is evaluated at machine speed.
    \item \textbf{Loss Function:} Squared error of side lengths and variance minimization for equidistance.
    \item \textbf{Clean Output:} Ensure all NumPy types are cast to native Python lists and floats.
\end{enumerate}
\textbf{\texttt{</think>}}

\begin{minted}[
    frame=lines,
    framesep=2mm,
    baselinestretch=1.2,
    bgcolor=black,
    formatcom=\color{white},
    fontsize=\footnotesize,
    style=monokai
]{python}
import numpy as np
from scipy.optimize import basinhopping
from numba import jit

@jit(nopython=True)
def geometry_loss(x):
    # Anchor B at (0,0), A on Y-axis, C on X-axis
    A = np.array([0.0, x[0]]); B = np.array([0.0, 0.0])
    C = np.array([x[1], 0.0]); P = np.array([x[2], x[3]])
    loss = 0.0
    def d2(p1, p2): return np.sum((p1 - p2)**2)

    # Constraints: Side lengths 3, 4, 5
    loss += (d2(A, B) - 3.0**2)**2 * 100.0
    loss += (d2(B, C) - 4.0**2)**2 * 100.0
    loss += (d2(A, C) - 5.0**2)**2 * 100.0

    # Incenter Property: P equidistant to lines AB, BC, AC
    def dist_line_sq(p, u, v):
        num = ((v[0]-u[0])*(u[1]-p[1]) - (v[1]-u[1])*(u[0]-p[0]))**2
        den = d2(u, v) + 1e-12
        return num / den

    d_ab = dist_line_sq(P, A, B)
    d_bc = dist_line_sq(P, B, C)
    d_ac = dist_line_sq(P, A, C)
    
    # Penalize variance between distances to ensure equidistance
    loss += (d_ab - d_bc)**2 * 50.0 + (d_bc - d_ac)**2 * 50.0

    # Inequality: P must be inside (Winding order check)
    cp1 = (B[0]-A[0])*(P[1]-A[1]) - (B[1]-A[1])*(P[0]-A[0])
    cp2 = (C[0]-B[0])*(P[1]-B[1]) - (C[1]-B[1])*(P[0]-B[0])
    cp3 = (A[0]-C[0])*(P[1]-C[1]) - (A[1]-C[1])*(P[0]-C[0])
    if cp1 > 0: loss += cp1**2 * 1000.0
    if cp2 > 0: loss += cp2**2 * 1000.0
    if cp3 > 0: loss += cp3**2 * 1000.0
    return loss

x0 = np.array([3.0, 4.0, 1.0, 1.0])
res = basinhopping(geometry_loss, x0, niter=100, minimizer_kwargs={"method": "L-BFGS-B"})
sol = res.x.tolist()
points = {"A": [0.0, sol[0]], "B": [0.0, 0.0], "C": [sol[1], 0.0], "P": [sol[2], sol[3]]}
circles = {"P": abs(float(sol[3]))}
\end{minted}

\vspace{2mm}
\textbf{Example 2: Constructive Geometry} \\
\textbf{Problem:} Construct an equilateral triangle ABC with side length 6 and its Circumcircle O.

\vspace{2mm}
\textbf{\texttt{<think>}}
\begin{enumerate}
    \item \textbf{Anchors:} Place A at $(0, h)$, B at $(-3,0)$, and C at $(3,0)$.
    \item \textbf{Height:} For side $6$, $h = 6\sqrt{3}/2$.
    \item \textbf{Circumcenter:} In an equilateral triangle, the circumcenter O is the centroid. $O = (A+B+C)/3$.
    \item \textbf{Radius:} O is the distance from A to any vertex.
\end{enumerate}
\textbf{\texttt{</think>}}

\begin{minted}[
    frame=lines,
    framesep=2mm,
    baselinestretch=1.2,
    bgcolor=black,
    formatcom=\color{white},
    fontsize=\footnotesize,
    style=monokai
]{python}
import numpy as np
side = 6.0
h = side * np.sqrt(3) / 2.0
B = np.array([-side/2, 0.0]); C = np.array([side/2, 0.0]); A = np.array([0.0, h])
O = (A + B + C) / 3.0
radius_O = np.linalg.norm(A - O)
points = {"A": A.tolist(), "B": B.tolist(), "C": C.tolist(), "O": O.tolist()}
circles = {"O": float(radius_O)}
\end{minted}

\vspace{2mm}
\textbf{FINAL AUDIT CHECKLIST:} \\
{[...]} 
\end{promptbox}

\subsection{Reproducibility details}
\label{app:reproducibility}

This appendix consolidates compute and hyperparameter information needed to reproduce the main-text experiments. The companion repository (\url{https://github.com/Huawei-AI4Math/PyGeoX}) contains full configuration files, exact commands, and dataset splits.

\paragraph{Base model.} Qwen3-8B for all main-text rows of Table~\ref{tab:gcs_merged}. SFT on Llama-3.1-8B-Instruct was also attempted (LoRA, $r{=}8$, $\alpha{=}32$) but produced near-zero performance across all reward configurations, reflecting insufficient math and instruction-following capacity (see Section~\ref{sec:discussion}).

\paragraph{SFT hyperparameters (Qwen3-8B).} Training framework: \texttt{ms-swift}. LoRA fine-tuning ($r{=}8$, $\alpha{=}32$, dropout $0.05$) on all linear projections. Learning rate $1 \times 10^{-5}$, cosine decay, 10\% warmup, weight decay $0.01$, Adam with $(\beta_1, \beta_2) = (0.9, 0.95)$, gradient clipping at $1.0$. Effective batch size: gradient accumulation of 32 steps with per-device batch size 2 across 8 GPUs. Maximum sequence length $18{,}196$ tokens, 2 epochs, bf16 precision with FlashAttention-2. Sample weighting via \texttt{dataset\_weighted} loss scaling. Checkpoints evaluated every 50 steps; the best checkpoint by validation loss was selected. Training-data volumes equalized across reward variants by adjusting the number of epochs.

\paragraph{RL hyperparameters (Qwen3-8B).} Framework: \texttt{OpenRLHF}. Algorithm: GRPO (\texttt{group\_norm} advantage estimator). Training data: $\sim$10k medium-difficulty problems sampled from the procedural corpus. Learning rate $5 \times 10^{-6}$, cosine decay with minimum LR, 3\% warmup. KL coefficient $0.01$ (k3 estimator), asymmetric PPO clipping $\epsilon \in [0.2, 0.3]$, discount factor $\gamma = 1.0$, PTX regularization $0.05$, Adam with $(\beta_1, \beta_2) = (0.9, 0.95)$, gradient clipping at $1.0$. $G = 8$ rollouts per problem, maximum generation length $8{,}192$ tokens, train batch size $144$, micro-train batch size $2$. Reward clipped to $[-10, 10]$. Inference via vLLM with 4 engines. Composite reward (Eq.~\ref{eq:main_reward}) with shaping weight $w = 6.0$, temperature $T = 0.1$, success-bonus $R_{\text{bonus}} = 4.0$, degeneracy penalty cap $4$, and tolerance $\|\mathbf{r}\|_2^2 < 10^{-3}$. Sandbox execution timeout 90\,s with \texttt{numpy}, \texttt{scipy}, \texttt{sympy}.

\paragraph{Reward variants.} The five reward formulations used in the ablation are: SAR ($\sum_i \exp(-r_i/T)$ with $T{=}0.1$); MSE ($\exp(-\|\mathbf{r}\|_2^2 / T_{\text{mse}})$ with $T_{\text{mse}}{=}10$, chosen to maximize reward spread on partially correct solutions); Sparse ($\mathbb{I}[\|\mathbf{r}\|_2^2 < 10^{-3}]$); SAR+S+D and MSE+S+D add the sparse success bonus and degeneracy penalty from Eq.~\ref{eq:main_reward} to SAR and MSE respectively.

\paragraph{Seeds and runs.} All runs use seed 42. Each reward variant was trained with a single seed to isolate the effect of reward design. Evaluation is one-shot (single sample per problem) for both PyGeoX-Bench and PyGeoX-Wild.

\paragraph{Data generation.} The 100k-problem training corpus, PyGeoX-Bench (300 problems), and the procedural pipeline are all included in the release. PyGeoX-Wild's 86 problems are adapted from a published middle-school geometry benchmark (citation withheld for anonymity; full citation will appear in the camera-ready version).

\subsection{Gradient informativeness analysis}
\label{app:gradient_analysis}

To empirically validate the Outlier Gradient Masking analysis (Section~3), we evaluated all five reward formulations on 3{,}893 partially correct solutions from the training corpus, i.e., problems where the base model satisfies some but not all constraints. To isolate the dense component, we report the normalized reward from each formulation \emph{without} the sparse bonus or degeneracy penalty, so all values lie in $[0, 1]$. For GRPO, absolute reward values matter less than relative gaps: the algorithm ranks rollouts within each group, so rewards must discriminate between solutions of varying quality. Table~\ref{tab:gradient_info} reports the fraction of samples in three reward regimes.

\begin{table}[h]
\centering
\caption{Reward discrimination on 3{,}893 partially correct solutions. All values are the normalized dense component only (no sparse bonus or degeneracy penalty), mapped to $[0, 1]$. SAR places $97\%$ of samples in the useful range where relative differences are preserved; MSE collapses $60\%$ to near-zero, making most partially correct solutions indistinguishable to GRPO.}
\label{tab:gradient_info}
\begin{tabular}{lccc}
\toprule
\textbf{Reward} & \textbf{Collapsed $\sim$0 ($<0.1$)} & \textbf{Useful $[0.1, 0.9]$} & \textbf{Collapsed $\sim$1 ($>0.9$)} \\
\midrule
SAR              & 2.9\%  & \textbf{97.1\%} & 0.0\%  \\
SAR+S+D          & 4.8\%  & \textbf{95.2\%} & 0.0\%  \\
MSE+S+D          & 63.5\% & 36.5\%          & 0.0\%  \\
MSE              & 60.4\% & 19.0\%          & 20.6\% \\
Sparse           & 100.0\%& 0.0\%           & 0.0\%  \\
\bottomrule
\end{tabular}
\end{table}

SAR's per-constraint decomposition ensures that partially correct solutions receive rewards that reflect the degree of constraint satisfaction, preserving relative gaps that GRPO can exploit for ranking. MSE's global-norm aggregation collapses the majority of these solutions to near-zero, making them indistinguishable and confirming the Outlier Gradient Masking failure mode identified in Section~3.

\subsection{Frontier-model context table}
\label{app:frontier_context}

Table~\ref{tab:frontier_context} reports the difficulty-calibration numbers for several frontier closed-source systems referenced in Section~\ref{subsec:comparative}.

\begin{table}[h]
\centering
\caption{Frontier closed-source systems evaluated zero-shot on PyGeoX, 
included to calibrate task difficulty. All systems evaluated via 
official APIs in January 2026 using vendor-default inference settings. 
The controlled comparison for our reward-design contribution is 
Table~\ref{tab:gcs_merged}.}
\label{tab:frontier_context}
\begin{tabular}{l l c c c c}
\toprule
\textbf{Model} & \textbf{Inference} & \textbf{PyGeoX (E)} & \textbf{PyGeoX (M)} & \textbf{PyGeoX (H)} & \textbf{PyGeoX-Wild} \\
\midrule
DeepSeek-V3.2  & default   & 0.52 & 0.39 & 0.25 & 0.64 \\
Proprietary-A  & default   & 0.53 & 0.57 & 0.37 & 0.65 \\
Proprietary-B  & reasoning & 0.55 & 0.54 & 0.38 & 0.77 \\
Proprietary-C  & reasoning & 0.61 & 0.49 & 0.51 & 0.81 \\
\midrule
\textbf{Qwen3-8B-RL (Ours)} & RL-trained & 0.62 & 0.50 & \textbf{0.41} & 0.66 \\
\bottomrule
\end{tabular}
\end{table}

\subsection{Examples of PyGeoX geometry problems and Qwen-3-8B-RL generations}
\label{sec:appendix_examples}

This section presents three representative geometric problems from PyGeoX-Bench and the training set, accompanied by their respective diagrammatic illustrations. For each case, we provide the full reasoning traces and generated Python code produced by the Qwen-3-8B-RL model with the SAR reward. These examples are selected to demonstrate the model's capabilities across a spectrum of difficulty: Figure \ref{fig:example_easy} illustrates a successful step-by-step geometric construction for an easy problem, while Figures \ref{fig:example_medium} and \ref{fig:example_hard} depict failure cases that highlight the model's limitations when navigating highly constrained medium and hard problems.

\begin{figure}[H]
    \centering
        \includegraphics[width=0.20\columnwidth, valign=c]{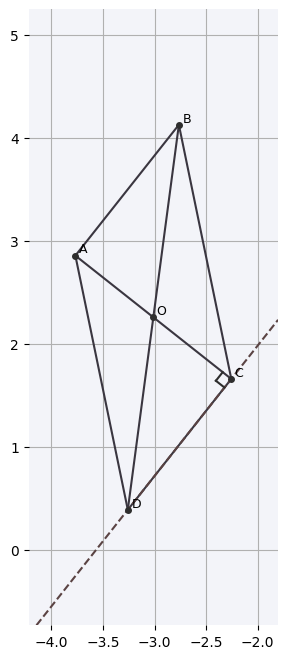}
        \hspace{0.5cm} 
        \includegraphics[width=0.30\columnwidth, valign=c]{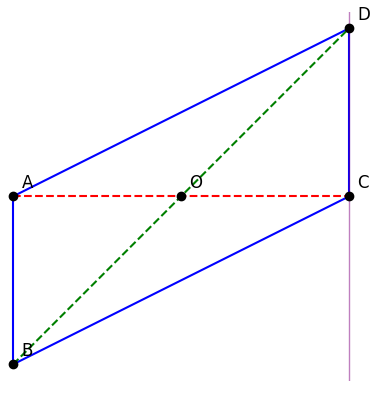}
        \caption{Ground truth diagram image for the easy difficulty example generated by PyGeoX (left) and Qwen3-8B-RL (right). The reward was 10.}
    \label{fig:example_easy}
\end{figure}

\begin{promptbox}{Easy problem example: 1obj\_2rel\_2extra\_gen0067}
\small
\textbf{Problem:} The diagram contains points $A, B, C, D, O$. There is an isosceles trapezoid $ABCD$. There is a line segment $AC$. There is a line segment $BD$. There is a line $CD$. Line $AC$ and line $BD$ intersect at point $O$. Line $CD$ is perpendicular to line $AC$.

\vspace{2mm}
\textbf{\texttt{<think>}}
\begin{enumerate}
    \item[] [...] Let me assume point $O$ is the intersection of the diagonals $AC$ and $BD$. Since $CD$ is perpendicular to $AC$, then the intersection point $O$ lies on both diagonals.
    \item[] [...] Let me set $AC$ as the x-axis passing through the origin $O$. Let $A=(-a, 0)$ and $C=(a, 0)$.
    \item[] [...] Since $CD \perp AC$, $CD$ must be vertical. So $D$ is at $(a, b)$.
    \item[] [...] For $ABCD$ to be an isosceles trapezoid with bases $AB$ and $CD$, $AB$ must also be vertical (parallel to $CD$).
    \item[] [...] Diagonals of an isosceles trapezoid intersect on the axis of symmetry. Since $O$ is the origin, $B$ must be symmetric to $D$ through the origin.
    \item[] [...] If $AC$ and $BD$ are diagonals intersecting at the origin, and $AB \parallel CD$, then $B$ must be $(-a, -b)$ to make $BD$ a straight line through the origin.
    \item[] [...] Let's check leg lengths: $BC^2 = (2a)^2 + b^2$ and $DA^2 = (2a)^2 + b^2$. They are equal. This works.
    \item[] [...] Assign $a=1, b=1$. Points: $A(-1,0), B(-1,-1), C(1,0), D(1,1), O(0,0)$.
\end{enumerate}
\textbf{\texttt{</think>}}

\begin{minted}[
    frame=lines,
    framesep=2mm,
    baselinestretch=1.2,
    bgcolor=black,
    formatcom=\color{white},
    fontsize=\footnotesize,
    style=monokai
]{python}
import numpy as np

# Define the points with coordinates based on the symmetric construction
points = {
    "A": [-1.0, 0.0],
    "B": [-1.0, -1.0],
    "C": [1.0, 0.0],
    "D": [1.0, 1.0],
    "O": [0.0, 0.0]
}

# No specific circles are mentioned in the problem, so circles is an empty dictionary
circles = {}
\end{minted}
\end{promptbox}

\begin{figure}[H]
    \centering
        \includegraphics[width=0.4\columnwidth, valign=c]{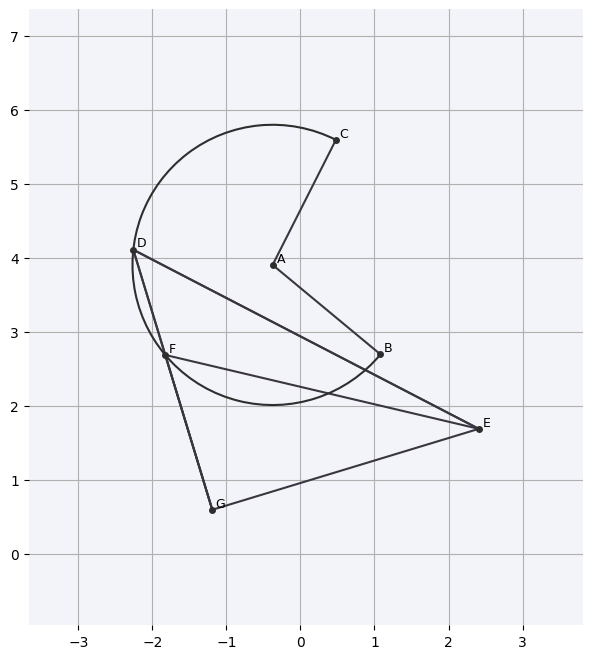} 
        \includegraphics[width=0.4\columnwidth, valign=c]{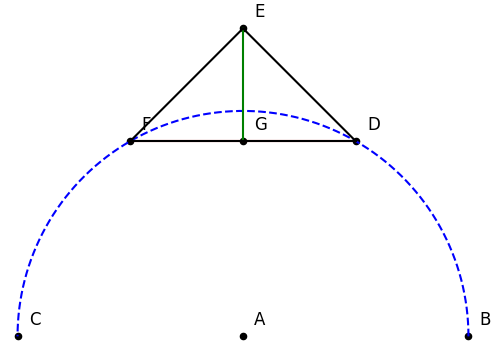}
        \caption{Ground truth diagram image for the medium difficulty example generated by PyGeoX (left) and Qwen3-8B-RL (right). The reward was 3.0.}
    \label{fig:example_medium}
\end{figure}

\begin{promptbox}{Medium problem example: 2obj\_4rel\_2extra\_gen0170}
\small
\textbf{Problem:} The diagram contains points $A, B, C, D, E, F, G$. There is a major arc with center $A$ and endpoints $B$ and $C$. There is a quadrilateral with vertices $D, E, F, G$. There is a line segment $DF$. There is a line segment $EG$. There is a line segment $DE$. There is a line segment $FG$. Point $D$ lies on the major arc with center $A$ and endpoints $B$ and $C$. Point $F$ lies on the major arc with center $A$ and endpoints $B$ and $C$. Points $D, F$, and $G$ are collinear. Line segment $DF$ is perpendicular to line segment $EG$.

\vspace{2mm}
\textbf{\texttt{<think>}}
\begin{enumerate}
    \item[] [...] Let me start by placing $A$ at the origin $(0,0)$ for simplicity. Let's set the radius to 2, so $B$ is at $(2,0)$ and $C$ is at $(-2,0)$.
    \item[] [...] Points $D$ and $F$ are on the major arc. I need $DF$ to be perpendicular to $EG$.
    \item[] [...] Let me pick $D$ and $F$ such that line $DF$ is horizontal. If $D$ is at $(1, \sqrt{3})$ and $F$ is at $(-1, \sqrt{3})$, both are on the circle ($1^2 + (\sqrt{3})^2 = 4$).
    \item[] [...] Line $DF$ is horizontal ($y = \sqrt{3}$). Since $D, F, G$ are collinear, $G$ must be on this line. Let me choose $G$ as the midpoint $(0, \sqrt{3})$.
    \item[] [...] $EG$ must be perpendicular to $DF$. Since $DF$ is horizontal, $EG$ is vertical.
    \item[] [...] $E$ must be on the vertical line $x=0$. Let's place $E$ at $(0, \sqrt{3} + 1)$.
    \item[] [...] Check: $DF$ slope is 0, $EG$ slope is infinite. They are perpendicular.
\end{enumerate}
\textbf{\texttt{</think>}}

\begin{minted}[
    frame=lines,
    framesep=2mm,
    baselinestretch=1.2,
    bgcolor=black,
    formatcom=\color{white},
    fontsize=\footnotesize,
    style=monokai
]{python}
import numpy as np

# Define points
A = [0.0, 0.0]
B = [2.0, 0.0]
C = [-2.0, 0.0]
D = [1.0, np.sqrt(3)]
E = [0.0, np.sqrt(3) + 1.0]
F = [-1.0, np.sqrt(3)]
G = [0.0, np.sqrt(3)]

points = {
    "A": A,
    "B": B,
    "C": C,
    "D": D,
    "E": E,
    "F": F,
    "G": G
}

circles = {}
\end{minted}
\end{promptbox}

\begin{figure}[H]
    \centering
        \includegraphics[width=0.4\columnwidth, valign=c]{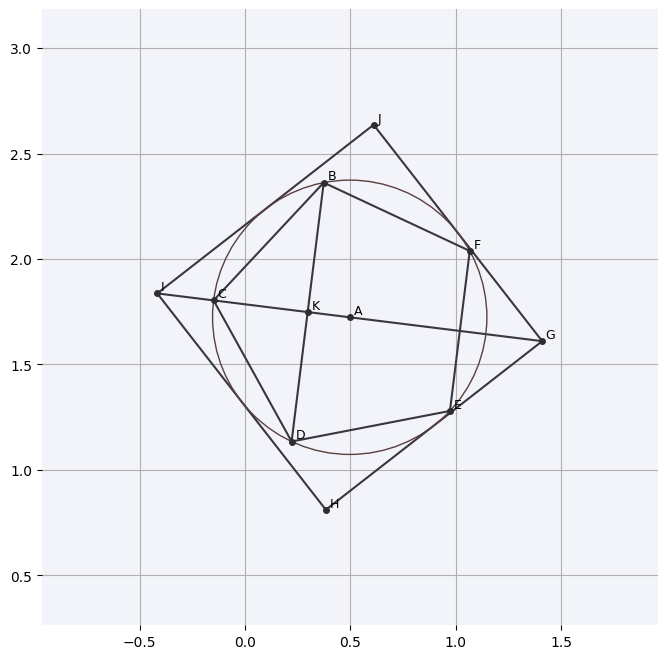} \includegraphics[width=0.4\columnwidth, valign=c]{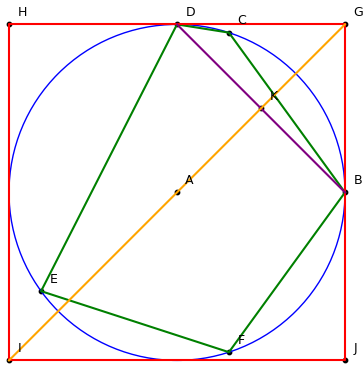}
        \caption{Ground truth diagram image for the hard difficulty example generated by PyGeoX (left) and Qwen3-8B-RL (right). The reward was 4.71.}
    \label{fig:example_hard}
\end{figure}

\begin{promptbox}{Hard problem example: 3obj\_5rel\_2extra\_gen0244}
\small
\textbf{Problem:} The diagram contains points $A, B, C, D, E, F, G, H, I, J, K$. There is a circle with center $A$. There is a regular pentagon with vertices $B, C, D, E, F$. There is a square with vertices $G, H, I, J$. There is a line segment $BD$. There is a line segment $GI$. Line $BD$ and line $GI$ intersect at point $K$. Points $B, K, D$ are collinear. Line $BD$ is perpendicular to line $GI$. The circle is the circumcircle of the regular pentagon. The circle is the incircle of the square.

\vspace{2mm}
\textbf{\texttt{<think>}}
\begin{enumerate}
    \item[] [...] The circle centered at $A$ is the circumcircle of the pentagon and the incircle of the square. Let radius be $r$.
    \item[] [...] Square vertices for incircle $r$ must be at distance $r\sqrt{2}$. I'll place the square axis-aligned: $G(r,r), H(-r,r), I(-r,-r), J(r,-r)$.
    \item[] [...] Line $GI$ connects $(r,r)$ and $(-r,-r)$, so its slope is $1$.
    \item[] [...] $BD$ intersects $GI$ at $K$ and is perpendicular. Thus, $BD$ must have slope $-1$.
    \item[] [...] I place $B$ at $(r, 0)$. For $BD$ to have slope $-1$, $D$ must be at $(0, r)$.
    \item[] [...] Checking regular pentagon vertices: If $B$ is at $0^{\circ}$, $D$ should be at $144^{\circ}$. However, $(0,r)$ is at $90^{\circ}$.
    \item[] [...] The regular pentagon condition conflicts with the slope $-1$ condition in this orientation.
    \item[] [...] Given the constraints, I will proceed with $D(0,r)$ to satisfy the perpendicularity and intersection constraints, even though it deviates from a standard regular pentagon vertex.
    \item[] [...] Intersection $K$ of $y=x$ and $y=-x+r$ is $(r/2, r/2)$.
\end{enumerate}
\textbf{\texttt{</think>}}

\begin{minted}[
    frame=lines,
    framesep=2mm,
    baselinestretch=1.2,
    bgcolor=black,
    formatcom=\color{white},
    fontsize=\footnotesize,
    style=monokai
]{python}
import numpy as np

# Define the radius
r = 1.0

# Points
A = [0.0, 0.0]

# Square vertices (incircle radius r)
G = [r, r]
H = [-r, r]
I = [-r, -r]
J = [r, -r]

# Pentagon vertices (circumcircle radius r)
# Note: D is placed at (0, r) to satisfy perpendicularity
B = [r, 0.0]
C = [np.cos(np.radians(72)), np.sin(np.radians(72))]
D = [0.0, r] # Adjusted to satisfy slope constraint
E = [np.cos(np.radians(216)), np.sin(np.radians(216))]
F = [np.cos(np.radians(288)), np.sin(np.radians(288))]

# Intersection point K
K = [r/2, r/2]

# Circles
circles = {"A": r}

# Points dictionary
points = {
    "A": A, "B": B, "C": C, "D": D, "E": E, "F": F,
    "G": G, "H": H, "I": I, "J": J, "K": K
}
\end{minted}
\end{promptbox}

\end{document}